\documentclass[twoside,11pt]{article}
\usepackage{textcomp}
\usepackage{blindtext}
\usepackage{bbm}
\usepackage{bm}
\usepackage{comment}
\usepackage{amsmath}
\usepackage{bbm}
\usepackage{subcaption}
\usepackage{algorithm2e}
\usepackage{todonotes}
\usepackage{arydshln}
\usepackage{tikz}
\usetikzlibrary{backgrounds,fit,matrix,shapes.geometric, patterns.meta, positioning, arrows}
\usepackage{pgfplots}
\usepackage{mathrsfs}
\pgfplotsset{compat=1.15}
\usepackage{geometry}
\definecolor{brickred}{rgb}{0.8, 0.25, 0.33}
\definecolor{cb-burgundy}{RGB}{146,   0,   0}
\definecolor{cb-brown}      {RGB}{146,  73,   0}
\definecolor{cb-custom-green}{rgb}{0.01, 0.75, 0.24}
\definecolor{darkslategray}{rgb}{0.18, 0.31, 0.31}
\definecolor{antiquefuchsia}{rgb}{0.57, 0.36, 0.51}
\definecolor{cb-blue}       {RGB}{ 0, 109, 219}
\definecolor{cb-custom-orange}{RGB}{250, 148, 0}
\definecolor{cb-custom-purple}{RGB}{157, 100, 211}
\definecolor{cb-custom-dark-green}{RGB}{37, 139, 40}
\definecolor{ffffff}{rgb}{1,1,1}
\newcommand{\emptyfootnote}{\let\thefootnote\relax\footnotetext}
\usetikzlibrary{arrows.meta}
\usetikzlibrary{fit,backgrounds,positioning,calc}
\tikzset{%
  >={Latex[width=2mm,length=2mm]},
            base/.style = {rectangle, rounded corners, draw=black,
            minimum width=4cm, minimum height=1cm,
            text centered, font=\sffamily},
    test/.style={
       rectangle,
       fill=white, 
       font=\sffamily , 
       draw=darkslategray, very thick,
       text width=2.3cm,
       minimum height=3.5em,
       text centered,
       },
    test_b/.style={
       rectangle,
       fill=white, 
       font=\sffamily , 
       draw=darkslategray, dotted,
       minimum height=2em,
       text centered,
       },
        test_b2/.style={
       rectangle,
       fill=white, 
       font=\sffamily , 
       draw=darkslategray, dotted,
       minimum height=2em,
       text centered,
       },
       test_c/.style={
       rectangle,
       fill=white, 
       font=\sffamily , 
       draw=darkslategray, dotted,
       text width=7.3cm,
       minimum height=25em,
       minimum width=45em,
       text centered,
       },
     source/.style={
       minimum width=4cm,
       rounded corners,
       font=\sffamily , 
       draw=cb-brown, very thick,
       minimum height=3em,
       text width=3.5cm,
       inner sep=2.5pt,
       text centered,
       },
       syn/.style={source, 
       fill=cb-custom-green,
       draw=darkslategray, very thick,
       minimum width=2.5em,
       text width=2.5cm,
       },
       aug/.style={source, 
        fill={antiquefuchsia},
       draw=byzantium, very thick,
       minimum width=2.5cm,
       text width=6.2cm,
       },               
       nag/.style={source, 
       fill={brickred},
       draw=cb-burgundy, very thick,
       minimum width=2.5em,
       text width=2.5cm,
       },   
       pretrained/.style={test, 
       draw=brickred, dotted, thick, 
       fill = brickred!10,
       text width=4cm, 
       minimum width = 12em, 
       minimum height = 10em},
       llm/.style={test, 
       draw=black, dotted, thick, 
       minimum width = 2em, 
       minimum height = 6em},
       transfer/.style={test, 
       draw=brickred, dotted, thick, 
       text width=4cm, 
       minimum width = 12em, 
       minimum height = 10em},
       llm/.style={test, 
       draw=black, dotted, thick, 
       minimum width = 2em, 
       minimum height = 6em},
}

\newcommand{\wfc}{\texttt{WFC}}

\newcommand{\da}{\texttt{DA}}
\newcommand{\demonic}{\textit{demonic}}

\def\s{\mathcal{S}}
\newcommand{\Pb}{\mathbb{P}}
\newcommand{\sqrtp}[1]{\sqrt[\leftroot{0} \uproot{5} p]{#1}}

\newcommand{\normp}[1]{\left\| #1\right\|_p}
\newcommand{\abs}[1]{\left|#1\right|}

\newcommand{\customparagraph}[1]{\textit{#1}}

\usepackage[left]{lineno}
\usepackage{booktabs}

\usepackage{jmlr2e}

\usepackage{lastpage}
\jmlrheading{26}{2025}{1-\pageref{LastPage}}{3/25; Revised 10/25}{10/25}{25-0485}{Thibaud Leteno, Michael Perrot, Charlotte Laclau, Antoine Gourru and Christophe Gravier}

\ShortHeadings{Fair Classifier via Transferable Representations}{Leteno, Perrot, Laclau, Gourru and Gravier}
\firstpageno{1}

\begin{document}

\hypersetup{hidelinks}


\title{Fair Text Classification via Transferable Representations}

\author{\name Thibaud Leteno \email thibaud.leteno@univ-st-etienne.fr \\
       \addr Université Jean Monnet Saint-Etienne, \\
       CNRS, Institut d'Optique Graduate School, \\ 
       Laboratoire Hubert Curien UMR 5516, F-42023, \\ Saint-Etienne, France
       \AND
       \name Michael Perrot \email michael.perrot@inria.fr \\
       \addr Univ. Lille, Inria, CNRS, Centrale Lille, \\ UMR 9189 - CRIStAL, F-59000, 
       Lille, France
       \AND
       \name Charlotte Laclau \email charlotte.laclau@telecom-paris.fr \\
       \addr LTCI, Télécom Paris
       \\ Institut Polytechnique de Paris, France
        \AND
       \name Antoine Gourru \email antoine.gourru@univ-st-etienne.fr \\
       \addr Université Jean Monnet Saint-Etienne, \\
       CNRS, Institut d'Optique Graduate School, \\ 
       Laboratoire Hubert Curien UMR 5516, F-42023, \\ Saint-Etienne, France
       \AND
       \name Christophe Gravier \email christophe.gravier@univ-st-etienne.fr \\
       \addr Université Jean Monnet Saint-Etienne, \\
       CNRS, Institut d'Optique Graduate School, \\ 
       Laboratoire Hubert Curien UMR 5516, F-42023, \\ Saint-Etienne, France
       }
\editor{Manuel Gomez-Rodriguez}

\maketitle

\begin{abstract}
Group fairness is a central research topic in text classification, where reaching fair treatment between sensitive groups (e.g., women and men) remains an open challenge. We propose an approach that extends the use of the Wasserstein Dependency Measure for learning unbiased neural text classifiers.
Given the challenge of distinguishing fair from unfair information in a text encoder, we draw inspiration from adversarial training by inducing independence between representations learned for the target label and those for a sensitive attribute. We further show that domain adaptation can be efficiently leveraged to remove the need for access to the sensitive attributes in the data set we cure. We provide both theoretical and empirical evidence that our approach is well-founded.

\end{abstract}


\section{Introduction}


Machine learning algorithms have become increasingly influential in decision-making processes that significantly impact our daily lives. 
One of the major challenges that has emerged in research, both academic and industrial, concerns the fairness of these models, that is, their ability to treat individuals and groups equitably without causing prejudice or discrimination.
As more researchers work to overcome these shortcomings, the first problem is to define what \emph{fairness} is. This definition may hardly be consensual \citep{han-etal-2023-fair} or is at least difficult to establish, as it depends on situational and cultural contexts \citep{fiske2017prejudices}. In this work, we focus on group fairness (that we will refer to as fairness for simplicity), which prevents predictions related to individuals from being based on sensitive attributes such as gender or ethnicity. 
We then adopt common metrics for assessing group fairness in practice, which are based on the notion of disparate impact referenced in legal frameworks across several countries.\footnote{for example, GPDR, Article 22 \citep{GDPR2016a} and AI Act \citep{AIAct}, Recital 27 in the European Union, Title VII of the 1964 Civil Rights Act \citep{act1964civil} in the United States of America.} This type of metric considers a predictive model fair if its outcomes remain consistent across groups of individuals defined by sensitive attributes.
%

In this article, we focus on the problem of fairness in the domain of Natural Language Processing (NLP)~\citep{
li2023survey, chu2024fairness} and more specifically for text classification as it is one of the most ubiquitous tasks in our society, with prominent examples in medical and legal domains~\citep{demner2009can} or human resources~\citep{jatoba2019evolution}, to name a few. For more general overviews of fairness in machine learning systems, we refer the interested readers to \cite{caton2024fairness, barocas-hardt-narayanan}.  
Initially, works in text classification rely on text encoders, which are parameterized and learned functions that map tokens (arbitrary text chunks) into a latent space of controllable dimension, usually followed by a classification layer. Built upon the Transformers architecture~\citep{vaswani2017attention}, popular Pre-trained Language Models (PLMs) such as BERT~\citep{devlin2018bert} 
leverage self-supervised learning to train the text encoder parameters. These PLMs are further fine-tuned for the supervised task at hand. 
More recently, with the advent of powerful decoder-based models, practitioners started to prompt those models for classification tasks \citep{dubey-2024-evaluating,ruan-etal-2024-large}.
While many studies already report biases in NLP systems~\citep{sun2019mitigating, hutchinson2020social, tan2019assessing, liang2021towards, bender2021dangers}, these issues become even more significant with the advent of public-ready AI-powered NLP systems. 
As mentioned above, recent developments in NLP, such as prompting-based models, raise questions about ensuring fairness in text classification. \citet{atwoodinducing} highlight the limitations of prompting for fairness control, whereas regularization-based methods achieve better fairness-performance trade-offs. Meanwhile, \citet{roccabruna-etal-2024-will} evaluate multiple large decoder-based models 
alongside RoBERTa \citep{liu2019roberta} on temporal relation classification, finding that RoBERTa outperforms all the decoder-based models for this task.
However, other approaches leverage powerful decoder models to generate embeddings for various tasks, including text classification, as seen with SFR-Embedding-2\_R \citep{SFR-embedding-2} or NV-Embed-v2 \citep{lee2024nv} 
both built on Mistral-7B \citep{jiang2023mistral}.
While some recent works adopt this embedding-based strategy \citep{yang2024diagnosing}, others continue to rely on encoder-only architectures \citep{sturman-etal-2024-debiasing}.
For fairness control in text classification, this leaves two main approaches: incorporating fairness constraints into prompts or debiasing the model during fine-tuning. 
Our work is part of this latter setting. \\

\noindent
\customparagraph{Contributions}
This paper extends our work on Wasserstein independence for text classification \citep{leteno-etal-2023-fair} to mitigate bias in text classifiers. We introduce an extensive theoretical analysis and present additional experimental results.
Our approach addresses bias directly in the latent space, making it applicable to any text encoder or decoder (e.g., BERT or Mistral). 
To proceed, we disentangle the neural signals encoding bias from those used for predictions. Disentanglement-based methods have primarily focused on images or tabular data \citep{Jang_2024_CVPR, NEURIPS2019_1b486d7a}. In this paper, we introduce an approach tailored to NLP and capable of handling less-explored scenarios, including continuous sensitive attributes and regression tasks. 
Our method overcomes a major shortcoming of prior studies that rely on access to the sensitive attributes during training - regulations, such as GDPR \citep{GDPR2016a}, impose more stringent requirements for the collection and use of protected attributes, which can, in certain cases, pose constraints on some methodologies. 
In the following, we demonstrate that our approach tackles this issue by learning from simple data sets, such as toy data sets, to transfer knowledge and enable fair classification even when sensitive attributes are not available in the deployment data.
 
In a nutshell, our goal is to reduce the dependency between predictions and sensitive attributes to improve fairness. To achieve this, we minimized the 
Wasserstein Dependency Measure \citep{ozair2019wasserstein} between the hidden representations of two neural networks: one for the end-task classification and one for predicting the sensitive attributes. This requires approximating several measures relative to the initial objective of independence between the classifier and the sensitive attribute. 
In this paper, we establish the theoretical validity of these approximations. First, we examine the relation between the chosen dependency measure and various fairness metrics. Second, we derive an upper bound on the transfer of sensitive attributes, supporting the use of predicted sensitive attributes when the real ones are unavailable. Finally, we justify the use of latent representations and provide guarantees on this approximation. We further validate our approach empirically by comparing it to state-of-the-art methods and evaluating different variations of our architecture. \\

\noindent
\customparagraph{Organization of the paper}
The rest of this paper is organized as follows. Section \ref{related_work} presents recent advances related to our proposition. Section \ref{preliminaries} discusses our motivation, provides the background knowledge to understand our contributions, and presents our first results that establish the relation between fairness and the Wasserstein Dependency Measure. Section \ref{proposition} proceeds with the theoretical framework of the proposed approach and its analysis. 
Section \ref{sec:implementation} provides the description of the proposed approach and the algorithmic details of the implementation. 
Section \ref{sec:protocol} introduces the setting of our experiments, and Section \ref{sec:results} presents the experiments and their interpretations. We present our conclusions and research perspectives in Section \ref{sec:conclusion} and end the paper with a section dedicated to the limitations of our contributions.

\section{Related Works}
\label{related_work}

Recent work on fairness in NLP has focused on fair text classification with adversarial methods \citep{beutel2017data, zhang2018mitigating,elazar-goldberg-2018-adversarial,madras2018learning, TORRES2024100092} being widely investigated. \citet{han-etal-2021-diverse,han-etal-2021-decoupling} suggest using multiple discriminators, each learning distinct hidden representations or applying adversarial training across domains. Other contributions enforce fairness through balanced training \citep{han2021balancing}, batch selection \citep{rohfairbatch}, or by integrating fairness metrics, such as Equality of Opportunity, directly into the objective function \citep{shen-etal-2022-optimising, shen2022does}. However, these methods rely on access to sensitive attribute annotations during training, limiting their practical applicability. In this work, we overcome this constraint while providing strong theoretical guarantees. 

Next, we focus on related work that considers settings where sensitive attributes are unavailable, followed by fairness approaches based on dependency measures and theoretical guarantees. \\

\noindent
\customparagraph{Sensitive attribute access for fairness mitigation}
To address their absence, proxy models have been proposed to enhance fairness. Other approaches circumvent the use of sensitive attributes during training or inference by leveraging related features \citep{zhao_toward_2O22}, knowledge distillation \citep{chai2022fairness}, adversarial reweighted learning \citep{lahoti2020fairness}, proxy features \citep{gupta2018proxy}, or perturbations \citep{awasthi2020equalized}. However, \citet{kenfack2023fairness} recently highlighted the risks associated with proxy-sensitive attributes, which may exacerbate the fairness-accuracy trade-off.
Domain adaptation has also been explored as a means to address fairness in data sets lacking demographic information. \citet{schumann2019transfer} employ adversarial learning to enforce fairness in the source domain while predicting domain membership, while \citet{coston2019fair} propose loss reweighting to mitigate the absence of sensitive attributes in either domain. Our approach follows this line of research, specifically addressing the lack of sensitive attributes in the target domain. By working in the representation space to minimize divergence between domains, we aim to ensure that the classifier trained on the source domain treats both domains equivalently. \\

\noindent
\customparagraph{Fair classification with dependency measures} 
The Wasserstein distance has been increasingly used to enforce fairness constraints in machine learning. For instance, \citet{risser2022tackling} and \citet{jiang2020wasserstein} apply it to measure the discrepancy between the distributions of predictions conditionally on groups defined by the sensitive attribute. 
Although effective, these approaches are limited to categorical sensitive attributes and mainly favor conditional independence. In contrast, we propose to exploit the Wasserstein dependency measure, which captures the dependence between the joint distribution of the hidden output representations and the sensitive attribute, and the product of their marginals. This distinction allows us to assess and mitigate bias at a more fundamental level, ensuring that the learned representations themselves do not encode sensitive information. Our approach is inspired by \citet{ozair2019wasserstein}, which uses Wasserstein's dependency measure to improve representation learning for images. However, while their work focuses on improving feature representations for downstream tasks, we incorporate sensitive attributes into the estimation process to promote fairness. 

Another related approach in NLP is proposed by \citet{cheng2021fairfil}, which maximizes the mutual information between sentence representations and their augmented counterparts to remove sensitive information from inputs. However, as noted by \citet{shen2022does} and \citet{cabello2023independence}, this does not guarantee the independence between predictions and sensitive attributes. Our method differs by explicitly minimizing the dependency between representations of the same sentence processed by two different encoders, ensuring that predictions remain unaffected by sensitive attributes.

Additionally, our work shares conceptual similarities with \citet{nam2020learning}, which addresses bias in image data. However, instead of focusing on reweighting samples to counteract biases in a secondary model, we employ the Wasserstein distance to quantify and minimize the dependency between the representations learned by two models. More recently, \citet{iskander-etal-2024-leveraging} also seeks to mitigate disparities but relies on task-specific representations and KL divergence to enforce distributional uniformity across groups.

\noindent
\customparagraph{Theoretical guarantees in fairness}
Most fairness mitigation techniques are evaluated on test sets that may not fully represent real-world deployment scenarios \citep{Dunkelau2020FairnessAwareML, hort2024bias}. 
This highlights the need for theoretical guarantees to ensure the reliability of mitigation approaches with respect to fairness metrics. 
Several works provide such guarantees, often focusing on post-training corrections. For instance, \citet{woodworth2017learning} propose a post-hoc correction method with guarantees on classifier performance and prediction disparities across sensitive attributes. \citet{denis2024fairness} derive distribution-free fairness guarantees, while \citet{chzhen2020fair} establish fairness bounds dependent only on the dimensionality of the unlabeled data set. 

On the other hand, \citet{celis2019meta} develop a meta-learning framework to obtain an optimally fair classifier with respect to algorithmic complexity, and \citet{mcnamara2017provably} show that learned representations can satisfy both group and individual fairness criteria. 
Finally, a closely related work is \citet{Gupta_Ferber_Dilkina_VerSteeg_2021}, who consider Mutual Information to measure the dependency between representations, providing fairness guarantees based on this latter. They derive an upper bound on the Demographic Parity measure via the Mutual Information between latent representations and the sensitive attributes, as well as bounds on the Mutual Information between classification labels and conditional latent representations. However, unlike our approach, they do not provide guarantees on the dependency between the classification labels and sensitive attributes.

\section{Wasserstein Dependency Measure and Group Fairness}
\label{preliminaries}
This section introduces the notations used throughout the paper, along with the definitions of key fairness metrics and the Wasserstein Dependency Measure ($I_W$). We then present our first result, establishing a link between two popular group fairness metrics and $I_W$.

\subsection{Notations}\label{subsec:notation}

We consider a corpus of $n$ triplets $\{(x_i,y_i,a_i)\}_{i=1}^n$, where $x_i\in \mathcal{X}$ is a short document or a sentence, $y_i \in \mathcal{Y}$ is a label and $a_i \in \mathcal{A}$  is either a \textit{sensitive} attribute, such as gender, ethnicity or age, or represents intersectional groups of several sensitive attributes. In this paper, we assume that $\mathcal{Y}$ and $\mathcal{A}$ are discrete spaces, and we will often abuse notations such that $y\in\mathcal{Y}$ and $a\in\mathcal{A}$ represent either a target label or a vector representation obtained through one-hot encoding. 
The embeddings (or representations) are obtained thanks to an encoding function, $Enc$, that maps words into numeric values.
The objective is to predict outcomes  $y$ for a given input $x$ by estimating the conditional distribution $p(Y| X=x)$. To this end, we learn a scoring function $\pi_y: \mathcal{X} \rightarrow \mathcal{P}(\mathcal{Y})$ where $\mathcal{P}(\mathcal{Y})$ is the set of probability distributions over $\mathcal{Y}$. Given $\pi_y(x)$, the actual prediction is denoted by $\hat{y}$ and corresponds to the label predicted as most likely. 
For instance, in a social network context, one can learn a classifier to predict whether a message is toxic. This prediction could inform decisions such as banning the message or its author from the platform.

In modern NLP applications, deep classification often follows a two-step approach: the scoring function $\pi$ is expressed as $\pi_y = h_y \circ Enc$, where $Enc(x) \in \mathbb{R}^d$ maps a text $x$ into a low-dimensional embedding space, and $h_y$, typically a simple neural network layer with a softmax activation serves as the classification layer.

\subsection{Group Fairness}

Our goal is to learn fair models and we focus on two main definitions of fairness. On the one hand, we consider demographic parity \citep{hardt2016equality} which is defined, for a desirable outcome $y$ and a sensitive attribute $a$, as
\begin{align}\label{eq:dp}
    \mathbf{DP}_{a,y} ={}& \Pb\left(\hat{Y} = y \;\middle|\; A = a\right) - \Pb\left(\hat{Y} = y\right).
\end{align}
On the other hand, we consider equality of opportunity \citep{hardt2016equality} which is defined, for an outcome $y$ and a sensitive attribute $a$, as
\begin{align}\label{eq:eo}
       \textbf{EO}_{a,y} = \mathbb{P}(\hat{Y}=Y|Y=y, A=a)-\mathbb{P}(\hat{Y}=Y|Y=y).
\end{align}

\subsection{Wasserstein Dependency Measure}
Mutual Information (MI) is an information-theoretic metric that measures the statistical dependence or the amount of information shared between two variables. 
For two random variables $U\sim p(U)$ and $V\sim p(V)$ that takes values in 
$\mathcal{U}$ and $\mathcal{V}$, respectively, the MI is defined as the KL-divergence between the joint distribution $p(U,V)$ and the product of the marginal distributions $p(U)p(V)$ 
\begin{equation*}
\begin{split}
\text{MI}(U,V) &= \text{KL}(p(U,V)\Vert p(U)p(V)). 
\end{split}
\end{equation*}

Early works in fair classification introduced the idea that fairness can be improved by reducing the Mutual Information (MI) between the classifier's output, $\hat{Y}$, and the sensitive attribute, $A$ \citep{kamishima2012, zemel13}. Specifically, enforcing Demographic Parity (DP) corresponds to minimizing the MI between these two random variables, ensuring that $\hat{Y}$ is independent of $A$. Similarly, Equalized Odds (EO) can be formulated as minimizing the MI between $A$ and $\hat{Y}$ conditionally on the true label $Y$, ensuring that predictions remain independent of the sensitive attribute within each outcome class.

However, MI is known to be intractable for most real-life scenarios and has strong theoretical limitations as outlined by \citet{pmlr-v108-mcallester20a}. Notably, it requires an exponential number of samples in the value of the MI to build a high confidence lower bound, and it is sensitive to small perturbations in the data sample. To overcome this issue,  \citet{ozair2019wasserstein} propose a theoretically sound dependency measure, the \textit{Wasserstein Dependency Measure} ($I_W$), based on the Wasserstein 1-distance

\begin{equation*}
I_W(U,V) = W_1(p(U,V), p(U)p(V)).
\end{equation*}
Using the Kantorovich-Rubinstein duality, it can also be expressed as
\begin{equation}
\label{eq:kantorovich_rubinstein_equation}
    I_W(U,V) = \sup_{||f||_L \leq 1} \mathbb{E}_{U,V \sim p(U,V)} [f(U,V)] - \mathbb{E}_{U \sim p(U),V \sim p(V)} [f(U,V)],
\end{equation}

\noindent where $||f||_L \leq 1$ is the set of all 1-Lipschitz functions. 
The Wasserstein distance has been efficiently used in many machine learning applications \citep{frogner2015learning, courty2014domain, torres2021survey} and a particularly interesting one is that of fair machine learning \citep{jiang2020wasserstein,silvia2020general,gordaliza2019obtaining,laclau2021all}. 
We present the Wasserstein distance in Appendix \ref{wasserstein_for_fairness} along with technical lemmas relevant for our proposition.

\subsection{Connection with Group Fairness}
\label{sec:group_fairness}

In this section, we show a connection between the Wasserstein Dependency Measure and the two group fairness measures we consider. Hence, in the next lemma, we show that a linear combination of Demographic Parity or Equality of Opportunity for all possible values of $a$ and $y$ are equivalent to the Wasserstein Dependency Measure between well-chosen random variables. This result is reminiscent of the result of \citet{Gupta_Ferber_Dilkina_VerSteeg_2021}, who showed a connection between group fairness and mutual information.

\begin{lemma}[Group fairness and Wasserstein Dependency Measure.\label{thm:boundedgroupfairness}]
Let $I_W$ be the Wasserstein dependency measure, and $A$, $Y$, $\hat{Y}$ be random variables corresponding to the sensitive attribute, the true label, and the predicted label, respectively. Let $\normp{\cdot}$ be the ground metric for the Wasserstein 1-distance. We have that
\begin{align*}
    I_W(\hat{Y},A) ={}& \frac{\sqrt[p]{2}}{2} \sum_{a \in \mathcal{A}} \mathbb{P}(A=a) \sum_{y\in\mathcal{Y}} \left|\textbf{DP}_{a,y}\right| \;\text{,} \\
    I_W((\hat{Y}=Y)|Y=y,A|Y=y) ={}& \sqrt[p]{2}\sum_{a \in \mathcal{A}} \mathbb{P}(A=a|Y=y) \left|\textbf{EO}_{a,y}\right| \;\text{,}
\end{align*}
with $|.|$ denoting the absolute value.
\end{lemma}

\begin{proof}
\nolinenumbers
    The proof is provided in Appendix \ref{app:group_fairness}.
\end{proof}

This lemma shows that minimizing the Wasserstein Dependency Measure between well-chosen random variables is a sound way to minimize Demographic Parity or Equality of Opportunity. This motivates the regularization of a learning algorithm by $I_W(\hat{Y}, A)$ to improve the fairness of text classifiers.

\section{Predictive and Sensitive Information Approximations}
\label{proposition}

To improve classifier fairness, we aim to minimize the Wasserstein Dependency Measure ($I_W$) between the sensitive attribute $A$ and the label predictions $\hat{Y}$. However, this optimization presents several challenges, notably having access to the sensitive attributes and requiring to differentiate a signal that went through an argmax function to obtain the label predictions. 

To address these, we first approximate the sensitive attribute labels using their predicted values, $\hat{A}$, obtained from a neural network. Then, instead of working directly with $\hat{Y}$ and $\hat{A}$, we use their hidden representations, denoted as $Z_y$ and $Z_a$, from the corresponding neural networks to overcome the non-differentiability of the argmax function. We also provide guarantees on these approximations.
This leads to the following optimization objective for learning a fair text classifier:
\begin{equation}\label{eq:loss} \arg \min \mathcal{L}(Y,h_y(Enc(X_{y}))) + \beta ~I_W(Z_y,Z_a), \end{equation} where $I_W(Z_y,Z_a) = W_1(p(Z_y,Z_a),p(Z_y)p(Z_a))$. Here, $Z_y$ and $Z_a$ represent the hidden representations from two Multi-Layer Perceptrons (MLPs): one for classification and one for the proxy model introduced in Section \ref{sec:def_demonic}. 
The function $\mathcal{L}$ ensures the classifier achieves high accuracy on $Y$ (e.g., we consider the cross-entropy for binary classification), while the second term encourages fairness by constraining the learned representations. The hyperparameter $\beta \in \mathbb{R^+}$ controls the balance between accuracy and fairness, as the two objectives may converge at different speeds.

\smallskip

We refer to this approach as Wasserstein Fair Classification (\wfc). Details on its implementation are provided in Section \ref{sec:implementation}.

\subsection{Definition of the Demonic Model}
\label{sec:def_demonic}

In the following, we use a surrogate model, referred to as the \demonic~model, for predicting the sensitive attribute $A$ without requiring explicit observation of attributes at training time.
To proceed, we assume a similar architecture as for predicting the labels: we learn a scoring function $\pi_a = h_a \circ Enc$ which, given an example $x$, outputs a probability distribution over $\mathcal{A}$, with $h_a$ a fixed classification function predicting $A$. The predicted sensitive attribute is then $\hat{a}$ and corresponds to the most likely sensitive attribute according to $\pi_a$. 
Consequently, we propose to consider $I_W(\hat{Y}, \hat{A})$ instead of $I_W(\hat{Y}, A)$ to approximate the dependency between the predictions and the sensitive attributes.
In the next theorem, we study this approximation and show that it is close to the original measure while being dependent on the \demonic~model performance.

\begin{lemma}
\label{lemma_hat_approx}
Let $\hat{Y},\hat{A},A$ be random variables that correspond to the predicted label, predicted sensitive attribute, and true sensitive attribute, respectively. Let $\normp{\cdot}$ be the ground metric for the Wasserstein 1-distance. Then, we have that
\begin{align*}
    I_W(\hat{Y},A) \leq I_W(\hat{Y},\hat{A}) + 2 \sqrt[p]{2} \Pb(A\neq\hat{A}).
\end{align*}
\end{lemma}

\begin{proof}
\nolinenumbers
    The proof is provided in Appendix~\ref{app:bound_prediction}.
\end{proof}
This lemma shows that replacing $A$ by $\hat{A}$ is sound when the latter is an accurate estimate of the former, that is, when $\Pb(A\neq\hat{A})$ is small. In the next theorem, we combine this result with a standard generalization result to show that this remains valid in the finite sample regime. The proof is provided in Appendix \ref{app:bound_in_domain}.

\begin{theorem}
\label{theorem_in_domain_hat_approx}
Let $\hat{A}, A \in \{0, 1\}$, and $\mathcal{H}$ be a hypothesis space of $VC$-dimension $d$. Let $\normp{\cdot}$ be the ground metric for the Wasserstein 1-distance. Assume that we have access to a training set of $m$ i.i.d. examples. Then, with probability at least $1-\delta$, we have $\forall h \in \mathcal{H}$
\begin{align*}
    I_W(\hat{Y}, A) \leq I_W(\hat{Y}, \hat{A}) + 2 \sqrt[p]{2} \left(\hat{\varepsilon} + \sqrt{\frac{4}{m} \left(dlog\frac{2em}{d}+log\frac{4}{\delta}\right)}\right),
\end{align*}
with $e$, the base of the natural logarithm and $\hat{\varepsilon}$ the empirical risk of the demonic model. 
\end{theorem}

\noindent
\customparagraph{Remark} This bound indicates that minimizing $I_W(\hat{Y},\hat{A})$ allows to minimize $I_W(\hat{Y}, A)$. However, it is tight when the \demonic~model is accurately predicting the sensitive attributes. 
In other words, with an accurate \demonic~model, the bound on the error rate is low and the bound tends to the estimate $I_W(\hat{Y}, \hat{A})$. In the perfect case, where the \demonic~model achieves perfect predictions, the bound is simply $I_W(\hat{Y}, \hat{A})$. Moreover, with input data of sufficient size, the bound on the error rate $\varepsilon$ gets lower.
We will consider the case where the \demonic~model is trained on data out of the domain (transfer learning scenario) later in Section \ref{sec:cross_domain_setting}. 
Note that we can easily generalize to multi-label sensitive attributes by considering the Natarajan dimension \citep{natarajan1989learning} instead of the VC-dimension. Moreover, multiple sensitive attributes can be considered by looking at the groups' intersections to cast the problem as a multi-label one. Finally, continuous sensitive attributes can be handled by binning. However, finding the relevant
threshold for binning is a non-trivial problem. We present the different scenarios and solutions in Appendix \ref{app:scenarios_theorems_sa}.

\subsection{Demonic Model in Cross-Domain Settings}
\label{sec:cross_domain_setting}

Recall that $\hat{A}$ and the latent representations $Z_a$ are obtained through a proxy neural network trained to predict the sensitive attribute to tackle the lack of sensitive attribute annotation. As it, one can train $h_a$ on a different data set from the end-task one. 

Let us consider two data sets, the end-task data set (or target) $\mathcal{D}_{\mathcal{T}}$ and the side data set (or source) $\mathcal{D}_{\s}$. $\mathcal{D}_{\mathcal{T}} = \{x_{\mathcal{T},i},~ y_{\mathcal{T},i}\}_i^{n_{\mathcal{T}}}$ is composed of a set of features and labels, while $\mathcal{D}_{\s} = \{x_{\s,i},~ a_{\s,i}\}_i^{n_{\s}}$ is composed of a set of features and sensitive attributes. We assume that we are in the context of covariate shift: the feature distributions are different but the sensitive attribute distributions are similar ($\mathcal{A}_{\mathcal{T}} \approx \mathcal{A}_{\s}$). 

Then, we want to learn a mapping $\phi : X_{\s} \rightarrow X_{\mathcal{T}}$ and train the \demonic~model classification layer $h_a$ on the mapped $X_{\s}$: 

\begin{equation*}
    \min_{h_a, \phi}~ \mathcal{L}(h_a(Enc(X_{\s})),~ A_{\s}) + \Lambda(\phi(Enc(X_{\s})),~Enc(X_{\mathcal{T}})),
\end{equation*}
\noindent
with $\Lambda(\phi(Enc(X_{\s})),~Enc(X_{\mathcal{T}}))$ the measure of divergence between the embeddings of $\mathcal{X}_{\mathcal{T}}$ and $\mathcal{X}_{\s}$. Note that the encoder $Enc$ has to be the same for the source and target domains.

We provide experimental details in Section \ref{sec:training_the_demonic}.
Moreover, Theorem \ref{theorem_in_domain_hat_approx} can be adapted to this setting, only the approximation of the error rate of the \demonic~model changes.

\begin{theorem}
    \label{theorem_cross_domain_hat_approx}
    Assuming that $\hat{A}, A \in \{0, 1\}$. Assume that $\mathcal{D}_\mathcal{S}$ and $\mathcal{D}_\mathcal{T}$ are a source and a target distribution such that $\mathbb{P}_{\mathcal{D}_\mathcal{S}}(X=x) \neq \mathbb{P}_{\mathcal{D}_\mathcal{T}}(X=x)$ and $\mathbb{P}_{\mathcal{D}_\mathcal{S}}(A=a|X=x) = \mathbb{P}_{\mathcal{D}_\mathcal{T}}(A=a|X=x)$, that is assume a covariate-shift. Let $\normp{\cdot}$ be the ground metric for the Wasserstein 1-distance. Assume that $I_W(\hat{Y}, A)$ and $I_W(\hat{Y}, \hat{A})$ are computed on the target distribution and let $\varepsilon_{\s} = \mathbb{P}_{\mathcal{D}_\mathcal{S}}(\hat{A} \neq A)$, $\varepsilon_{\mathcal{T}} = \mathbb{P}_{\mathcal{D}_\mathcal{T}}(\hat{A} \neq A)$, then we have that
    \begin{align*}
            I_W(\hat{Y}, A) \leq I_W(\hat{Y}, \hat{A}) + 2 \sqrt[p]{2}\left(\varepsilon_{\s} + \frac{1}{2}d_{\mathcal{H}\Delta\mathcal{H}}(\mathcal{D}_{\s}, \mathcal{D}_{\mathcal{T}}) + \lambda\right), 
    \end{align*}
    where $d_{\mathcal{H}\Delta\mathcal{H}}(\mathcal{\tilde{D}}_{\s}, \mathcal{\tilde{D}}_{\mathcal{T}})$ is the $\mathcal{H}\Delta\mathcal{H}$-divergence between the marginal feature distributions $\mathcal{\tilde{D}}_{\s}$ and $\mathcal{\tilde{D}}_{\mathcal{T}}$ and $\lambda = \lambda_{\s} + \lambda_{\mathcal{T}}$ with $\lambda_{\s}$ and $\lambda_{\mathcal{T}}$ the errors of $h^* = argmin_{h\in\mathcal{H}}(\varepsilon_{\mathcal{T}}(h), \varepsilon_{\s}(h))$ with respect to $\mathcal{D}_{\s}$ and $\mathcal{D}_{\mathcal{T}}$ respectively.
\end{theorem}

\begin{proof}
\nolinenumbers
    This is a direct application of \citet[Theorem~2]{ben2010theory}.
\end{proof}

\noindent
\customparagraph{Remark} We can draw similar conclusions as for Theorem \ref{theorem_in_domain_hat_approx}. However, in this case, one must also consider the divergence between the domains, determinant to the success of the approximation. The closer the two domains are, the tighter the bound is. Therefore, if the \demonic~model decreases in accuracy due to the divergence between the source and target domains, the bound gets looser. Note that the $\mathcal{H}\Delta\mathcal{H}$-divergence between the source and target domains is restricted to the binary setting. To generalize to categorical sensitive attributes (and by extension to multiple or continuous sensitive attributes as in Theorem 3), one could consider a generalization of the $\mathcal{H}\Delta\mathcal{H}$-divergence as in \citep{sicilia2022pac}. Recall that we detail the different scenarios and solutions in Appendix \ref{app:scenarios_theorems_sa}.

\subsection{Using Latent Representations}
\label{sec:use_latent_rep}
In the previous section, we explained why using the Wasserstein Dependency Measure between the predicted labels and sensitive attributes, $I_W(\hat{Y}, \hat{A})$, instead of between the predicted labels and the true sensitive attributes, $I_W(\hat{Y}, A)$. Nevertheless, as such, we cannot consider this measure to regularize any training algorithms since the argmax operation producing the hard predictions ($\hat{Y}$) following the classification layer is not differentiable. Thus, instead of considering the network's final output, one can overcome this limitation by minimizing the $I_W$ between the latent representations of the networks $h_y$ and $h_a$, respectively referred to as $Z_y$ and $Z_a$. 
In Theorem \ref{theorem_latent_representations}, we show that the $I_W$ between the neural networks' representations is an upper bound of the $I_W$ between the predictions. 

\begin{theorem}\label{theorem_latent_representations}
    Let $\hat{Y},\hat{A}$ be random variables that correspond to the predicted label and predicted sensitive attribute, respectively. Assume that $h_y = \sigma_{\lambda}(f(Z_y))$ and $h_a = \sigma_{\lambda}(g(Z_a))$ where $\sigma_{\lambda}$ is the softmax function with temperature $\lambda$, $f$ and $g$ are both $L$-lipschitz with respect to the $p$-norm, and $Z_y$ and $Z_a$ are latent representations of the examples. Let $\normp{\cdot}$ be the ground metric for the Wasserstein 1-distance. For a given example $x$ with predicted label $\hat{y}$ and predicted sensitive attribute $\hat{a}$, let $\xi_y(x) = f(Z_y)_{\hat{y}} - \max_{y' \neq \hat{y}}f(Z_y)_{y'}$ and $\xi_a(x) = g(Z_a)_{\hat{a}} - \max_{a' \neq \hat{a}}g(Z_a)_{a'}$ be positive margins. Let $\delta = 1-\mathbb{P}(\xi_y(X)\geq \xi, \xi_a(X)\geq \xi)$ with $\xi > 0$. Let $\alpha=\sqrtp{2} \normp{\binom{|\mathcal{Y}|}{|\mathcal{A}|} - 1}(1-\delta)$ and $\iota=L(\abs{\mathcal{Y}}+\abs{\mathcal{A}})^{\abs{\frac{1}{2}-\frac{1}{p}}}$. Then, setting $\lambda = \frac{1}{\xi}\log\left(\frac{2\xi\alpha}{\iota I_W(Z_y, Z_a)}\right)$, we have that
    \begin{align*}
        I_W(\hat{Y},\hat{A}) \leq{}& 2I_W(Z_y, Z_a)\frac{\iota}{\xi} \left[1+\log \left( \max\left(4, \frac{2\xi\alpha}{\iota I_W(Z_y, Z_a)}\right) - 1\right)\right] + \sqrtp{2} \normp{\binom{|\mathcal{Y}|}{|\mathcal{A}|}-1} \delta.
    \end{align*}
\end{theorem}

\begin{proof}
\nolinenumbers
    The proof of a slightly sharper result, in particular when $I_W(Z_y, Z_a)$ is large, is provided in Appendix~\ref{app:bound_rep}. We present this simpler version here for better readability.
\end{proof}

\noindent
\customparagraph{Remark} This result suggests that minimizing $I_W(Z_y,Z_a)$ is a sound way to minimize $I_W(\hat{Y},\hat{A})$. 
The tightness of the bound depends mainly on the error introduced by the softmax and, more specifically, on two terms: $\xi$ and $\delta$.
The margin $\xi_y(x)$ (resp. $\xi_a(x)$) measures how dominant the predicted class is relatively to the others, i.e., it is large when $\hat{Y}$ in one-hot encoded form and $\sigma_\lambda(f(Z_y))$ are close. In other words, $\xi_y(x)$ (resp. $\xi_a(x)$) represents the confidence level of the classification model and $\xi$ represents the minimum expected confidence. The term $\delta$ is the proportion of examples for which this minimum confidence is not obtained by the model. We note that there is a trade-off between the first and the second term in the bound, depending on the value of $\xi$, as a high value of $\xi$ is likely to imply a large $\delta$ and vice versa. 

This result also indicates that for a given model, there is an optimal softmax temperature for inference, but, since the models are assumed given, it does not help us finding the optimal softmax temperature at training time. Furthermore, since the softmax is followed by an argmax function, the optimal temperature at inference has a limited impact. Finally, the way it is used here, temperature scaling is mainly seen as a technical tool and it might be possible to derive a similar result without it, thus it might not be necessary. For all these reasons, we do not further investigate this term experimentally.

\section{Implementation of Wasserstein Fair Classification}
\label{sec:implementation}

In this section, we present both the overall architecture of \wfc~and the implemented training strategy.

\subsection{Architecture of \wfc}

\begin{figure}[t]
    \centering
    \resizebox{\linewidth}{!}{
        \includegraphics[]{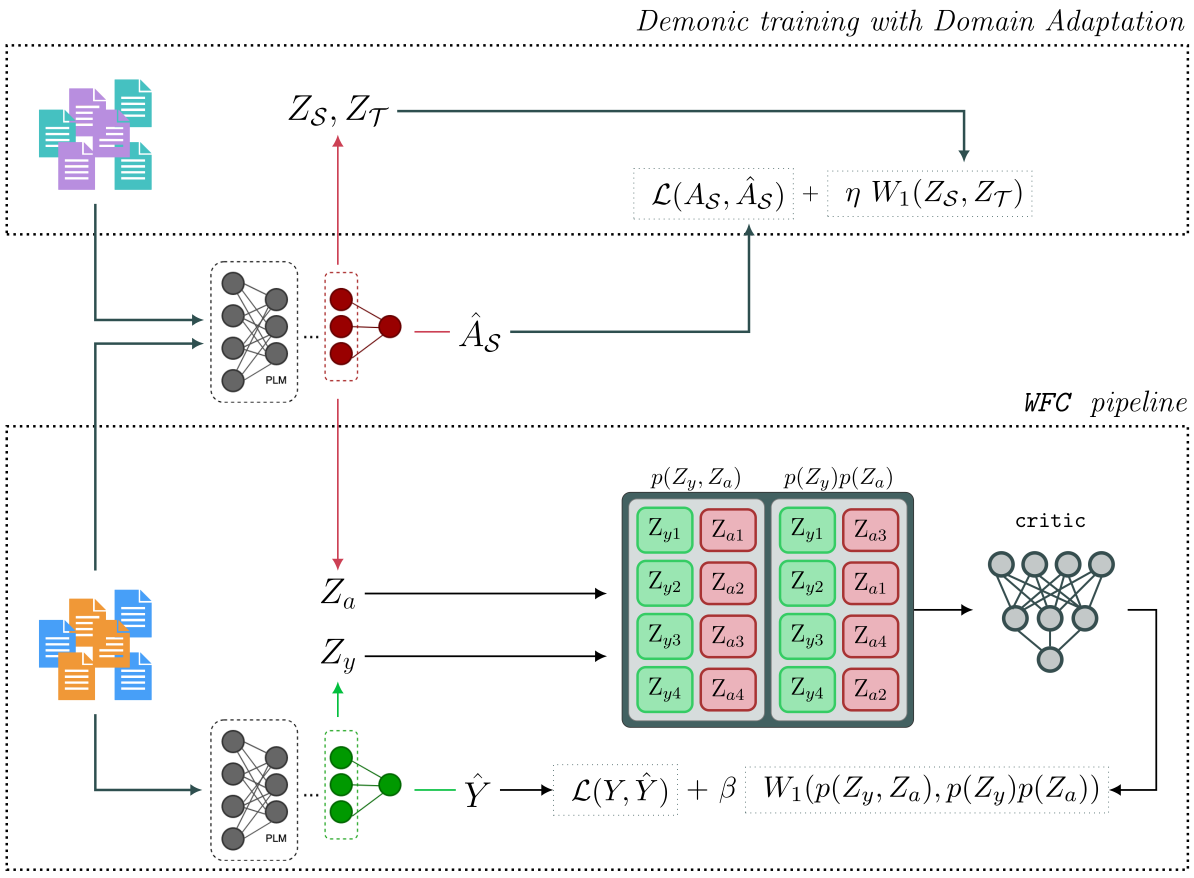}
    }
       \caption{Architecture of our method.
       The top part illustrates the pre-training of the \textit{demonic} model (red) with domain adaptation. The model is trained to predict the sensitive attribute on the source domain ($A_{\mathcal{S}}$) while minimizing the divergence between the hidden representations from the source and target domains ($Z_{\mathcal{S}}$ and $Z_{\mathcal{T}}$). 
       The bottom part describes the \wfc~ pipeline       
       for a batch of size 4, the demonic model is then frozen. The data representation on the right shows how we enforce dependency or independence between $Z_y$ and $Z_a$. 
       During inference, only the trained classifier (green) is retained to predict $Y$. 
    }
    \label{fig:architecture}
\end{figure}

The overall architecture of \wfc~is composed of three components: two classifiers and a critic (see Figure \ref{fig:architecture}). We recall that the architecture aims to minimize the loss function described in Equation \ref{eq:loss}. \\

\noindent
\customparagraph{Learning $Z_y$ and $Z_a$}Given a batch of documents along with their sensitive attribute, we start by generating a representation of each document using a pre-trained language model (PLM). These representations serve as input to two MLPs, which are trained to predict $A$ and $Y$, respectively. The first model, referred to as the \demonic~model, is pre-trained. The prediction $\hat{Y}$ outputted by the second MLP (in green in Figure \ref{fig:architecture}) is directly used to compute the first term of our objective function (see Equation \ref{eq:loss}).
Additionally, from a given hidden layer in each of the MLPs, we extract the hidden representation vectors, $Z_y$ and $Z_a$, which capture intermediate features relevant to their respective tasks. \\

\noindent
\customparagraph{Computing $I_W(Z_y, Z_a)$}
The second term of the loss is the $I_W$ between $Z_y$ and $Z_a$. To compute this latter, we use the following approximation \citep{arjovsky2017wasserstein}
\begin{equation}\label{eq:wass}
   \max_{\omega, ||C_w||_L \leq 1} \mathbb{E}_{Z_y,Z_a \sim p(Z_y,Z_a)} [C_{\omega}(Z_y,Z_a)] - \mathbb{E}_{Z_y \sim p(Z_y), Z_a \sim p(Z_A)} [C_{\omega}(Z_y,Z_a)].
\end{equation}
where $C_{\omega}$ is called the critic. Initially proposed for Wasserstein GAN (Generative Adversarial Network) \citep{arjovsky2017wasserstein}, the critic is a neural network used as an alternative to the GAN's discriminator to overcome this latter unstable training \citep{arjovsky2017towards}. It estimates the Wasserstein distance between real and fake distributions by outputting scores (the Wasserstein distance) instead of classifying samples as probabilities. 
The underlying idea is to use a neural network with a Lipschitz constraint on the weights to approximate the Kantorovich-Rubinstein formulation of the Wasserstein distance (Equation \ref{eq:kantorovich_rubinstein_equation}). The Lipschitz constraint is induced by weights clipping or gradient penalty \citep{gulrajani2017improved}.
We follow \citet{ozair2019wasserstein} 
who use it to compute the $I_W$ with this differentiable estimate of the Wasserstein distance.
To enforce the Lipschitz constraint, we clamp the weights to given values ($[-0.01, 0.01]$) at each optimization step.\footnote{We also tested some more recent improvements of Lipschitz constraint enforcement \citep{gulrajani2017improved, wei2018improving}. Interestingly, all lead to poorer performance.} 
For a batch of documents, the critic takes as input the concatenation of $Z_y$ and $Z_a$, and the concatenation of $Z_y$ and $Z_a$ randomly drawn from the data set (equivalent to $Z_y \sim p(Z_y), Z_a \sim p(Z_a)$). 
We then follow the training procedure introduced by \citet{arjovsky2017wasserstein}, which alternates maximizing Equation~\ref{eq:wass} in the \texttt{critic} parameters for $n_c$ iterations and minimizing Equation~\ref{eq:loss} for $n_d$ iterations in the $h_y$ classifier parameters. We add a comparison to \wfc$_{eo}$, where we compute and minimize the $I_W$  between instances that were well classified during the training. This allows us to compare optimizing directly DP vs. EO. 

\noindent
\customparagraph{Overall} The overview of the training process is detailed in Appendix~\ref{app:algorithm}.
The details of the MLPs used to parameterize each component, including their architecture, are given in Appendix~\ref{app:experimental_details}. We evaluate and optimize the hyperparameters for our models on a validation set, focusing on the MLP and Critic learning rates, the value of $n_d$ (number of batches used to train the main MLP), the layers producing $Z_a$ and $Z_y$, the value of $\beta$ and the value used to clamp the weights to enforce the  Lipschitz constraint. The values allowing us to obtain the optimal trade-off between accuracy and fairness (DTO, cf. Section \ref{sec:metric_protocol}) during this process are presented in Appendix~\ref{app:archi_t1}.
In all our experiments, and if not mentioned otherwise, the value of $\beta$ is set to 1.
Our implementation is available on Github: \url{https://github.com/LetenoThibaud/wasserstein_fair_classification}.

\subsection{Pre-training the \textit{Demonic} Model}
\label{sec:training_the_demonic}

\noindent
\customparagraph{Overview} We pre-train the \demonic~model, an MLP with a similar architecture as the previous classifier, to predict the sensitive attributes. Note that we do not update the \demonic~weights during the training phase of the main model. The benefits are twofold. First, unlike previous works~\citep{caton2020fairness}, we require only limited access to sensitive attribute labels during training, and we do not need access to the sensitive attributes at inference. This makes \wfc\ highly compatible with recent regulations (e.g., US Consumer Financial Protection Bureau). Second, the \demonic~model can be trained in a few-shot fashion if some examples of the training set are annotated with sensitive attributes. \\

\noindent
\customparagraph{Learning with a related data set} However, when no sensitive attributes are available in the training set, we replace the training data of the \demonic~ model with data from another domain (e.g., another data set) containing sensitive information for the same attribute. For example, for gender, we can leverage generated data sets, like the EEC data set \citep{kiritchenko2018examining}. This enables knowledge transfer between data sets, promoting fairness autonomy regardless of whether sensitive attributes are present in the data, as long as another data set with similar sensitive attributes exists. Finally, in most cases, sensitive attribute knowledge transfers easily between data sets without additional adjustments. However, when data set divergence is significant, domain adaptation techniques can be applied to ensure transfer quality. \\

\noindent
\customparagraph{Learning with domain adaptation}
If the training data set differs significantly from the end-task data set, we add a regularization term to the loss of the \demonic~model 
to train it with a double objective: 1) predicting the sensitive attribute and 2) generating representations from the source and target domains that are both close and informative for classification.
In practice, under the covariate shift assumption, we use the Wasserstein distance between the representations of the source and target data sets as a measure of divergence.
For domain adaptation, as for \wfc, a critic model estimates the Wasserstein distance between the source and target representations. We use this measure for domain adaptation as done in \citet{shen2018wasserstein}.
Note that while in \wfc~the Wasserstein distance is computed between the joint and the product of the marginal distributions of the representations to compute a measure of independence, here we compute it between the representations themselves.
Specifically, we compute the Wasserstein distance between the last hidden states of the model for both sets of representations (source and target). 
Therefore, if we consider the source and target domains, respectively $\mathcal{D}_{\s} = \{x_{\s,i},~ a_{\s,i}\}_i^{n_{\s}}$ and $\mathcal{D}_{\mathcal{T}} = \{x_{\mathcal{T},i},~ y_{\mathcal{T},i}\}_i^{n_{\mathcal{T}}}$, with $X_{\s}, X_{\mathcal{T}}$ the sets of input texts, $A_{\s}$ the sensitive attributes.
The objective of the \demonic~model, $h_a$, can be written as follows 

\begin{equation}\label{eq:loss_da}
    \arg \min \mathcal{L}(A_{\s},h_a(Enc(X_{\s})) 
    + \eta ~W_1(Z_{\s},~Z_{\mathcal{T}}),
\end{equation}
where $\mathcal{L}$ is the loss function aiming at maximizing the accuracy of $h_a$ on predicting $A$, and $Z_{\mathcal{S}},~ Z_{\mathcal{T}}$ are the hidden representations of the model respectively for $X_{\mathcal{S}}$ and $X_{\mathcal{T}}$. 

\section{Experimental Framework}
\label{sec:protocol}

In this section, we present the setting of our experiments, that is, the data sets and metrics we consider.

\subsection{Evaluation Metrics}
\label{sec:metric_protocol}

In this section, we introduce the metrics used to evaluate the performance of the models. For utility, we consider the balanced accuracy (Bal. Acc.) to handle class imbalance in the data.
For fairness, we recall in Section \ref{subsec:notation} the Equality of Opportunity (cf. Equation \ref{eq:eo}).
In our experiments, we consider binary sensitive attributes ($\mathcal{A}=\{0, 1\}$). For multi-class objectives (e.g. $\mathcal{Y}=\{1, \cdots, C\}$), one can aggregate EO scores over classes (formulated as the difference of true positive rates across sensitive groups). This measure is the TPR-parity (or TPR-GAP) score \citep{de2019bias, ravfogel-etal-2020-null} defined as follows
\begin{equation*}\label{eq:gap}
    \textbf{TPR-parity} = \sqrt{\frac{1}{|\mathcal{C}|} \sum_{c \in \mathcal{C}} (\textbf{TPR}_{1,c} - \textbf{TPR}_{0,c})^2},
\end{equation*}
with $\textbf{TPR}_{0,c}$ and $\textbf{TPR}_{1,c}$ the true positive rates for class $c$ and respectively sensitive group $0$ and $1$.
For clarity in the results' comparison with the accuracy score, we consider the following 
\begin{equation*}
    \textbf{Fairness} = (1 - \textbf{TPR-parity}) * 100.
\end{equation*}
The Fairness score indicates a perfectly fair model when equal to 100, and unfair when equal to 0. 
Additionally, as fairness often requires determining a trade-off such that reaching equity does not degrade the general classification performance, \citet{han2021balancing} proposed the Distance To Optimum (\textbf{DTO}) score. It measures the accuracy-fairness trade-off by computing the Euclidean distance from a model to an \emph{Utopia point} (point corresponding to the best accuracy and best fairness values across all the baselines). The goal is to minimize the DTO. Let consider the \emph{Utopia point} with coordinates $\{\textbf{accuracy}_u,~ \textbf{fairness}_u\}$ and the performance of a model at a given epoch $\{\textbf{accuracy}_m,~ \textbf{fairness}_m\}$, then we have

\begin{equation*}
    \textbf{DTO} = \sqrt{\left(\textbf{fairness}_u -  \textbf{fairness}_m \right)^2 + \left(\textbf{accuracy}_u -  \textbf{accuracy}_m \right)^2}.
\end{equation*}
Finally, we consider the \textbf{Leakage} metric, which measures the accuracy of a classification model trained to predict the sensitive attribute $A$ from the latent representations ($Z$) of another model. 
Let us consider two models, a classification model $h$ that we want to evaluate and another model $h_{leakage}$ trained to retrieve the sensitive information $A$ from the latent representations of $h$, $Z_h$. We consider a test set of size $n$ such that

\begin{equation*}
    \textbf{Leakage} = \left( \frac{1}{n} \sum_{i=0}^n \mathbbm{1}_{L}(Z_{hi}) \right) * 100~ \text{with}~ \mathbbm{1}_{L}(Z_h) = 
\begin{cases} 
1 & \text{if } h_{leakage}(Z_h) = A, \\
0 & \text{if } h_{leakage}(Z_h) \neq A.
\end{cases}
\end{equation*}
It measures the fairness \emph{of the latent representations themselves} and demonstrates representation unfairness when close to 100. 
We use the architecture presented in \citet{shen2022does}, that is an MLP with one hidden layer of size 100, a ReLU activation function, and a constant learning rate of 0.001. The optimizer used is Adam.

\subsection{Data Sets}
We employ two widely-used data sets to evaluate fairness in the context of text classification, building upon prior research \citep{ravfogel-etal-2020-null,han-etal-2021-diverse,shen2022does}. Both data sets are readily available in the FairLib library \citep{han2022fairlib}. \\

\noindent
\customparagraph{Bias in Bios \citep{de2019bias}.} This data set, referred to as ``Bios data set'' in the rest of the paper, consists of brief biographies from the common crawl associated with occupations (a total of $28$) and genders (male or female). As per the partitioning prepared by \citet{ravfogel-etal-2020-null}, the training, validation, and test sets comprise $257,000$, $40,000$, and $99,000$ samples, respectively. \\ 

\noindent
\customparagraph{Moji \citep{blodgett2016demographic}.} This data set contains tweets written in either ``Standard American English'' (SAE) or ``African American English'' (AAE), annotated with positive or negative polarity. We use the data set prepared by \citet{ravfogel-etal-2020-null}, which includes $100,000$ training examples, $8,000$ validation examples, and $8,000$ test examples. The target variable $Y$ represents the polarity, while the protected attribute corresponds to the ethnicity, indicated by the AAE/SAE attribute. 

\section{Results and Discussion}
\label{sec:results}

In this section, we consider three experimental axes to illustrate our method: 1) in-domain experiments compared to state-of-the-art methods, 2) cross-domain experiments, 3) analysis of the \wfc~ method. 

\subsection{Comparison with State-of-the-Art Methods}

Firstly, we compare our approach with state-of-the-art methods and different text encoders. \\

\begin{table}[t]
    \centering
        \resizebox{\textwidth}{!}{
        \begin{tabular}{lcccc}
        \hline
        \textbf{Model} & \textbf{Bal. Acc. $\uparrow$} & \textbf{Fairness $\uparrow$} & \textbf{DTO $\downarrow$} & \textbf{Leakage $\downarrow$} \\
        \hline
        *CE & $72.3 \pm 0.5^{\boldsymbol{\dagger}}$ & $61.2 \pm 1.4^{\boldsymbol{\dagger}}$ & $31.0$ & $87.9 \pm 3.3$ \\ \hdashline
        INLP + BERT$_{\text{ft}}$ & $73.3 \pm 0.0^{\boldsymbol{\dagger}}$ & $85.6 \pm 0.0^{\boldsymbol{\dagger}}$ & $8.49$ 
        & $86.7 \pm 0.6$ \\
        Adv + BERT$_{\text{ft}}$ & $75.6 \pm 0.4$ & $90.4 \pm 1.1$ & $4.03$ 
        & $78.8 \pm 6.0^{\boldsymbol{\dagger}}$ \\
        Gate + BTEO + BERT$_{\text{ft}}$ & $76.2 \pm 0.3^{\boldsymbol{\dagger}}$ & $90.1 \pm 1.30$ & $\boldsymbol{3.55}$ 
        & $100.0 \pm 0.0^{\boldsymbol{\dagger}}$ \\
        FairBatch + BERT$_{\text{ft}}$ & $75.1 \pm 0.6^{\boldsymbol{\dagger}}$ & $90.6 \pm 0.5$ & $4.47$  
        & $88.4 \pm 0.4^{\boldsymbol{\dagger}}$ \\
        EO$_{GLB}$ + BERT$_{\text{ft}}$ & $75.2 \pm 0.2$ & $90.1 \pm 0.4^{\boldsymbol{\dagger}}$ & $4.49$ 
        & $85.7 \pm 1.2$ \\
        DAFair + BERT$_{\text{ft}}$ & $\boldsymbol{79.5 \pm 0.2}^{\boldsymbol{\dagger}}$ & $73.1 \pm 1.1^{\boldsymbol{\dagger}}$ & $18.3$ & - \\ \hdashline
        Adv & $74.5 \pm 0.3^{\boldsymbol{\dagger}}$ & $81.5 \pm 2.0^{\boldsymbol{\dagger}}$ & $11.1$ 
        & - \\
        Gate + BTEO & $74.9 \pm 0.2^{\boldsymbol{\dagger}}$ & $86.2 \pm 0.3^{\boldsymbol{\dagger}}$ & $6.94$ 
        & - \\
        Con$_{dp}$ & $\textcolor{blue}{75.8 \pm 0.3}^{\boldsymbol{\dagger}}$ & $88.1 \pm 0.6^{\boldsymbol{\dagger}}$ & $4.96$ 
        & $\boldsymbol{\textcolor{blue}{54.2 \pm 0.9}}^{\boldsymbol{\dagger}}$ \\
        Con$_{eo}$ & $74.1 \pm 0.7^{\boldsymbol{\dagger}}$ & $84.1 \pm 3.0^{\boldsymbol{\dagger}}$ & $9.08$  
        & $80.1 \pm 4.2^{\boldsymbol{\dagger}}$ \\
       \wfc & $75.2 \pm 0.1$ & $\boldsymbol{\textcolor{blue}{91.4 \pm 0.3}}$ & $4.29$ 
        & $86.9 \pm 0.2$ \\
         \wfc$_{eo}$ & $ 75.1 \pm 0.1 $ & $ 91.0 \pm 0.8 $ & $ 4.39 $ & $ 85.9 \pm 0.2 $ \\
        \wfc~+ BTEO & $ 75.3 \pm 0.1 $ & $ 91.1 \pm 0.3 $ & $\textcolor{blue}{4.21}$ 
        & $ 87.2 \pm 0.5 $ \\
        \hline
        \end{tabular}}
        \caption{Results on Moji. For baselines, results are drawn from \citet{shen2022does}. We report the mean $\pm$ standard deviation over 5 runs. * indicates the model without fairness consideration, and - indicates that we cannot access the result. The best results are in bold, results in blue indicate the best results without fine-tuning BERT. $\dagger$ indicates statistical significant difference with \wfc~based on the Student's t-test.}
        \label{tab:task1_moji}
\end{table}

\noindent
\customparagraph{Baselines} 
First, we consider *CE, that is, the architecture without any form of regularization. Then, the baselines are INLP \citep{ravfogel-etal-2020-null}, the ADV method \citep{han-etal-2021-diverse}, FairBatch \citep{rohfairbatch}, GATE \citep{han2021balancing}, EO$_{GLB}$ \citep{shen-etal-2022-optimising} and Con, displaying the $dp$ and $eo$ versions \citep{shen2022does}.
\citet{shen2022does} extend some of the methods by rebalancing classes during training (+ BTEO) or fine-tuning a BERT model in addition to the trainable MLP (+ BERT$_{ft}$).
We also consider DAFair \citep{iskander-etal-2024-leveraging} in our baselines due to the proximity with our work as indicated in Section \ref{related_work}, and rerun their experiments with similar settings as the
authors on respective data sets to ensure comparable results with regards to the splits and seeds. If not mentioned otherwise, the results of the other baselines are drawn from \citet{han2022fairlib} and \citet{shen2022does}. We did not rerun it since we built our code on the Fairlib library made available by the authors. As such, the base architectures, evaluation protocols, seeds, and data splits are similar. \\

\noindent
\customparagraph{Setting} To compare our method against state-of-the-art approaches, we first use the representation generated by a base BERT model as an input to the MLPs. For Bios, the \demonic~MLP is trained on 1\% of the training set and obtains 99\% accuracy for predicting the sensitive attributes on the test set. Similarly, the \demonic~MLP obtains 88.5\% accuracy on Moji.
Except for the standard cross-entropy loss without a fairness constraint (CE) and the DAFair baseline, which we run ourselves, we report results from~\citet{shen2022does,han2022fairlib} as mentioned in \S Baselines. In our approach, embedding representations are derived from a fixed BERT model, with only the MLP weights being adjusted. We also evaluate the quality of our method under balanced training as in \citet{shen2022does}. \\

\begin{table}[t]
    \centering
    \resizebox{\textwidth}{!}{
        \begin{tabular}{lcccc}
        \hline
        \textbf{Model} & \textbf{Bal. Acc. $\uparrow$} & \textbf{Fairness $\uparrow$} & \textbf{DTO $\downarrow$} & \textbf{Leakage $\downarrow$} \\
        \hline
        *CE & $82.3 \pm 0.2$ & $85.1 \pm 0.8^{\boldsymbol{\dagger}}$ & $5.87$ & $98.0 \pm 0.0^{\boldsymbol{\dagger}}$ \\ \hdashline
        INLP + BERT$_{\text{ft}}$ & $82.3 \pm 0.0$ & $88.6 \pm 0.0^{\boldsymbol{\dagger}}$ & $2.61$ & $97.6 \pm 0.1^{\boldsymbol{\dagger}}$ \\
        Adv + BERT$_{\text{ft}}$ & $81.9 \pm 0.2^{\boldsymbol{\dagger}}$ & $ 90.6 \pm 0.5^{\boldsymbol{\dagger}}$ & $1.81$ & $88.6 \pm 4.6^{\boldsymbol{\dagger}}$ \\
        Gate + BTEO + BERT$_{\text{ft}}$  & $\boldsymbol{83.7 \pm 0.2}^{\boldsymbol{\dagger}}$ & $90.4 \pm 0.9^{\boldsymbol{\dagger}}$ & $\boldsymbol{0.40}$ & $100.0 \pm 0.0^{\boldsymbol{\dagger}}$ \\
        FairBatch + BERT$_{\text{ft}}$  & $82.2 \pm 0.1^{\boldsymbol{\dagger}}$ & $89.5 \pm 1.3$ & $1.98$ & $98.0 \pm 0.3^{\boldsymbol{\dagger}}$ \\
        EO$_{GLB}$ + BERT$_{\text{ft}}$ & $81.7 \pm 0.4^{\boldsymbol{\dagger}}$ & $88.4 \pm 1.0$ & $3.12$ & $97.2 \pm 0.5$ \\
        DAFair + BERT$_{\text{ft}}$ & $\boldsymbol{83.7 \pm 0.1}^{\boldsymbol{\dagger}}$ & $86.4 \pm 0.3^{\boldsymbol{\dagger}}$ & 4.40 & - \\ \hdashline
        Adv & $81.1 \pm 0.1^{\boldsymbol{\dagger}}$ & $87.3 \pm 0.9^{\boldsymbol{\dagger}}$ & $4.36$ & - \\
        Gate + BTEO & $79.4 \pm 0.1^{\boldsymbol{\dagger}}$  & $\boldsymbol{{\color{blue} 90.8 \pm 0.2}}^{\boldsymbol{\dagger}}$ & $4.30$ & - \\
        Con$_{dp}$ & $82.1 \pm 0.2^{\boldsymbol{\dagger}}$ & $84.3 \pm 0.8^{\boldsymbol{\dagger}}$ & $6.69$ & $\boldsymbol{{\color{blue} 76.3 \pm 1.5}}^{\boldsymbol{\dagger}}$ \\
        Con$_{eo}$ & $81.8 \pm 0.3^{\boldsymbol{\dagger}}$ & $85.2 \pm 0.4^{\boldsymbol{\dagger}}$ & $5.91$ & $84.9 \pm 3.4^{\boldsymbol{\dagger}}$ \\
        \wfc & ${\color{blue} 82.4 \pm 0.1}$ & $ 89.0 \pm 0.3 $ & $ 2.22 $ & $ 96.5 \pm 0.5 $ \\
        \wfc$_{eo}$ & $ 82.1 \pm 0.2^{\boldsymbol{\dagger}}$ & $ 89.0 \pm 0.2 $ & $ 2.42 $ & $ 97.4 \pm 0.3^{\boldsymbol{\dagger}}$ \\
        \wfc~+ BTEO & $ 82.3 \pm 0.2 $ & $ 89.1 \pm 0.3 $ & $ {\color{blue} 2.20} $ & $ 96.7 \pm 0.5 $ \\ 
        \hline
        \end{tabular}
        }
    \caption{Results on Bios. For baselines, results are drawn from \citet{shen2022does}. We report the mean $\pm$ standard deviation over 5 runs. * indicates the model without fairness consideration, - indicates that we do not have access to these results. The best results are in bold, results in blue indicate the best results without fine-tuning BERT. $\dagger$ indicates statistical significant difference with \wfc~based on the Student's t-test.}
    \label{tab:task1_bios}
\end{table}

\noindent
\customparagraph{Discussion}
We compare \wfc~with text classification baselines. For Moji (Table \ref{tab:task1_moji}), the accuracy of \wfc\  is higher than the accuracy of CE, and it is equivalent to competitors. Considering the fairness metrics, we outperform all baselines. 
Note that DAFair, related to our work with the KL-divergence as dependency measure, outperforms all baselines in terms of accuracy with a limited gain of Fairness.
For Bios (Table \ref{tab:task1_bios}), our method is competitive with the other baselines and ranks 4 out of 12 with BTEO and 5 without it in terms of accuracy-fairness trade-off (DTO). 
Especially, \wfc~has the second-best accuracy compared to baselines. 
Moreover, on both data sets, we obtain similar Leakage as the comparable baselines, we further discuss this score in Section \ref{sec:beta_leakage}.

Note that BERT is not fine-tuned during our training pipeline. This decision is based on several factors: first, fine-tuning BERT increases training complexity and may hinder convergence. Additionally, it makes our method flexible to any encoder or decoder architecture, regardless of size. 
However, among the baselines without BERT fine-tuning, we reach the lowest DTO, comparable to those obtained with methods that fine-tune BERT.

When comparing the versions of \wfc~ optimizing EO or DP and rebalancing classes, we report close results on the three approaches. Noting a slightly better DPO on the version optimizing DP (\wfc), we consider this version in the other experiments. Despite the better DTO of \wfc~ + BTEO, we do not choose it for the experiments in Sections \ref{sec:cross_domain} and \ref{sec:component_exploration} to evaluate the method without external influence.

Ultimately, compared to the baselines, our method demonstrates notable advantages, particularly its ability to achieve competitive performance without access to sensitive attributes in the training set. We assess this capability in the section \ref{sec:cross_domain}. In the next subsection, we explore an alternative model for generating the representations used by the classifier.

\subsubsection{Using recent decoder-based model}

\noindent
\customparagraph{Setting} State-of-the-art baselines use BERT representations. However, recent PLMs have surpassed BERT's performance. Additionally, many modern embedding models are based on a decoder architecture. Therefore, we assess the robustness of our method using representations from SFR-Embedding-2\_R
model\footnote{\url{https://huggingface.co/Salesforce/SFR-Embedding-2_R}} \citep{SFR-embedding-2} built on the Mistral model \citep{jiang2023mistral}. This model is ranked first on the MTEB benchmark\footnote{\url{https://huggingface.co/spaces/mteb/leaderboard}} \citep{muennighoff2022mteb} on July 8th, 2024, notably for the classification task. 
We realize this set of experiments on the Bios data set and exclude the Moji data set since we do not have access to the raw text, and that the embeddings depend on the DeepMoji model \citep{felbo2017using}. The \demonic~MLP is also trained on SFR-Embedding-2\_R's representations.
We compare our approach to the cross-entropy without regularization (CE), as well as the best baselines on BERT concerning fairness and accuracy (respectively, GATE and ADV). The approaches are evaluated with and without balanced training (BTEO). We realize hyperparameter tuning for all methods as described in Appendix \ref{app:details_task_4}.

\begin{table}[t]
    \centering
    \begin{tabular}{lcccc}
        \hline
        \textbf{Model} & \textbf{Bal. Acc. $\uparrow$} & \textbf{Fairness $\uparrow$} & \textbf{DTO $\downarrow$} & \textbf{Leakage $\downarrow$} \\
        \hline
        *CE & $\boldsymbol{85.5 \pm 0.09}^{\boldsymbol{\dagger}} $ & $ 86.1 \pm 0.36 ^{\boldsymbol{\dagger}}$ & $ 6.63 $ & $ 97.9 \pm 0.41 $ \\
        GATE & $ 85.3 \pm 0.23 $ & $ 83.5 \pm 0.60 ^{\boldsymbol{\dagger}}$ & $ 9.22 $ & $ 100.0 \pm 0.01 ^{\boldsymbol{\dagger}}$ \\
        GATE + BTEO & $ 84.4 \pm 0.14 ^{\boldsymbol{\dagger}}$ & $\boldsymbol{ 92.7 \pm 0.67 }^{\boldsymbol{\dagger}}$ & $ \boldsymbol{1.10} $ & $ 99.9 \pm 0.13 ^{\boldsymbol{\dagger}}$ \\
        ADV & $ 84.8 \pm 0.72 $ & $ 90.3 \pm 0.40 $ & $ 2.49 $ & $ 89.1 \pm 7.96 ^{\boldsymbol{\dagger}}$ \\
        ADV + BTEO & $ 84.3 \pm 0.07 ^{\boldsymbol{\dagger}}$ & $ 91.4 \pm 0.41 ^{\boldsymbol{\dagger}}$ & $ 1.74 $ & $ \boldsymbol{86.2 \pm 6.05} ^{\boldsymbol{\dagger}}$ \\
        \wfc & $ 85.2 \pm 0.02 $ & $ 90.0 \pm 0.21 $ & $ 2.74 $ & $ 97.8 \pm 0.41 $ \\
        \wfc~ + BTEO & $ 85.1 \pm 0.06 ^{\boldsymbol{\dagger}}$ & $ 90.0 \pm 0.25 $ & $ 2.75 $ & $ 97.8 \pm 0.34 $ \\
        \hline
        \end{tabular}
        \caption{SFR-Embeddings-2\_R Results Bios. We report the mean $\pm$ standard deviation over 5 runs. * indicates the model without fairness consideration. The best results are in bold. $\dagger$ indicates statistical significant difference with \wfc~based on the Student's t-test.}
        \label{tab:results_sfr}
\end{table}

\noindent
\customparagraph{Discussion} We evaluate the efficiency of our architecture on recent decoder-based models to generate the embedding representations and compare them with the best baselines on the BERT-encoding results. We perform this evaluation on the Bios data set as explained above and present results in Table \ref{tab:results_sfr}. We observe an improvement of both accuracy and fairness for all methods compared to the results with a BERT encoder. However, in this experiment, improving fairness comes at the cost of performance compared to the model without regularization (*CE). Among all baselines, ours enhances fairness while minimizing performance the less. In contrast, other baselines that improve fairness (GATE + BTEO, ADV, and ADV + BTEO) lead to a performance drop of up to one point.

\subsection{Cross-domain \wfc}
\label{sec:cross_domain}

We consider two experiments to assess the transfer of sensitive attributes: with and without the domain adaptation procedure. We conduct these experiments on Bios, as other data sets with gender annotations are already available, unlike AAE/SAE data sets for Moji. 

The main objective of this section is to evaluate the performance of \wfc~when the \demonic~is trained on other sources than the task data set. 

\subsubsection{Zero-shot cross-domain demonic training}

\noindent
\customparagraph{Setting} We consider two source data sets to train the \demonic~MLP without domain adaptation. The EEC data set \citep{kiritchenko2018examining} consists of 8,640 synthetic sentences in English for Sentiment Analysis. The Marked Personas (MP) data set \citep{cheng2023marked} is composed of 2,700 descriptions of individuals obtained using a generative procedure: we consider the \texttt{dv2} version. 
We then evaluate the \wfc~ pipeline with those  \demonic~MLP. When training on the EEC data set, we obtain, on average over 5 runs, 98.1\% of accuracy,  and 98.4\% on the MP data set. \\

\begin{table}[t]
        \centering
        \begin{tabular}{lccccc}
        \toprule
        \textbf{Data} & \textbf{Bal. Acc. $\uparrow$} & \textbf{Fairness $\uparrow$} & \textbf{DTO $\downarrow$} & \textbf{Leakage $\downarrow$} & \begin{tabular}{c}
            \textbf{Demonic} \\
            \textbf{Bal. Acc. $\uparrow$}
        \end{tabular} \\
        \midrule
        Bios 1\% & $ \textbf{82.4} \pm \textbf{0.1} $ & $ \textbf{89.0} \pm \textbf{0.3} $ & $ \textbf{2.22} $ & $ 96.5 \pm 0.5 $ & $\boldsymbol{99.0}$  \\
        EEC  &  $ 82.2 \pm 0.4 $ & $ 88.9 \pm 0.4 $ & $ 2.42 $ & $ 97.5 \pm 0.3 $ & 98.1 \\
        MP & $ \textbf{82.4} \pm \textbf{0.3} $ & $ 88.9 \pm 0.4 $ & $ 2.30 $ & $ \textbf{96.4} \pm \textbf{0.5} $ & 98.4 \\
        \bottomrule
        \end{tabular}
        \caption{Comparison between several scenarios for training the \demonic~model for prediction on Bios. We report the mean $\pm$ standard deviation over 5 runs.}
        \label{tab:diff_demon}
    \end{table}

\noindent
\customparagraph{Discussion}
Table \ref{tab:diff_demon} shows that when the source and target data sets are similar, we achieve results comparable to those obtained when pre-training is performed using the same data set. The average loss in accuracy and fairness is minimal, with the standard deviation causing the measurements to overlap. These results are promising for improving fairness, especially in situations where collecting sensitive data is not feasible or when only partial information is available. In the next subsection, we investigate when the divergence between the source and target is higher and consider domain adaptation to train the \demonic~model.

\subsubsection{Demonic training with domain adaptation}

\noindent
\customparagraph{Setting and protocol}
We begin by considering a variant of the MP data set for this experiment. A set of gendered words (listed in Appendix \ref{app:cross_domain}) is removed from the texts to increase the divergence with the Bios data set. 
Next, we train a \demonic~model on this data set with the values of regularization $\eta \in \{0.5, 1, 2\}$ on the domain adaptation term in Equation 
\ref{eq:loss_da} recalled below
\begin{equation*}
    \arg \min \mathcal{L}(A_{\s},h_a(Enc(X_{\s})) 
    + \eta ~W_1(Z_{\mathcal{S}},~Z_{\mathcal{T}}).
\end{equation*}
We run the pipeline for 15000 epochs; at each epoch, the critic is trained on 20 batches and the model on 5 batches. 
We assess different values for the learning rate on a validation set and obtain the following optimal learning rate: $1e^{-5}$.
We also compare to the baseline, which consists of training the \demonic~for 20 epochs on the source data set only, without any adaptation. For the baseline, the \demonic~model is optimized with the following objective
\begin{equation*}
    \arg \min \mathcal{L}(a,h_a(z_{source}))
\end{equation*}
Finally, we run the \wfc~ pipeline with the \demonic~obtained as in the previous set of experiments. \\

\begin{table}[t]
    \centering
    \begin{tabular}{ccc}
        \toprule
         Method & Accuracy on $\mathcal{S}$ & Accuracy on $\mathcal{T}$  \\ \midrule
        baseline & $ 65.3  \pm  3.23 $ & $ 75.3  \pm  13.9 $ \\
        $\eta = 0.5 $ & $ 75.0  \pm  0.00 $ & $ 96.5  \pm  0.94 $ \\
        $\eta = 1 $ & $ \boldsymbol{81.3  \pm  0.67} $ & $ \boldsymbol{98.0  \pm  0.24} $ \\
        $\eta = 2 $ & $ 75.0  \pm  0.00 $ & $ 95.9  \pm  0.27 $ \\
         \bottomrule
    \end{tabular}
    \caption{Performance of the \demonic~model trained with domain adaptation. Performance on the source is given for the best corresponding performance on the target set.}
    \label{tab:demonic_da_performance}
\end{table}

\noindent
\customparagraph{Cross-domain demonic performance}
As shown in Table \ref{tab:demonic_da_performance}, the domain adaptation procedure significantly improves the performance of the \demonic~model on the sensitive attributes predictions when the domains diverge. Note that the value of $\eta$ matters; with a lower $\eta$, the adaptation may be too weak to align the domains, whereas with a higher $\eta$, the regularization term may overly influence the classification term in the loss.

Interestingly, for the case of gender, when the most common expressions of gender are removed from the source but remain in the target domain, the procedure also helps to improve the performance of the \demonic~model on the source domain. Furthermore, we note better scores on the target domain than on the source. We hypothesize that the presence of the gender markers in the target domain makes the task easier, inducing better results.

Finally, it is interesting to note the variance on the \textit{baseline demonic}: while in some cases domain adaptation will not be necessary, the procedure ensures an efficient \demonic~ model without regard to the initial conditions of the optimization. \\

\begin{table}[h]
    \centering
    \resizebox{\textwidth}{!}{
        \begin{tabular}{lccccc}
        \hline
        \textbf{Model} & \textbf{Bal. Acc. $\uparrow$} & \textbf{Fairness $\uparrow$} & \textbf{DTO $\downarrow$} & \textbf{Leakage $\downarrow$} & \textbf{Demonic accuracy $\uparrow$} \\
        \hline
        *CE & $82.3 \pm 0.20$ & $85.1 \pm 0.80$ & $5.87$ & $98.0 \pm 0.00$ & - \\
        Baseline & $ 82.5 \pm 0.05 $ & $ 86.8 \pm 0.50 $ & $ 4.19 $ & $ 97.1 \pm 0.44 $ & $ 75.3  \pm  13.9 $ \\
        $\eta = 0.5$ & $ \boldsymbol{82.5 \pm 0.02} $ & $ 87.4 \pm 0.21 $ & $ 3.57 $ & $ \boldsymbol{96.6 \pm 0.36} $ & $ 96.5  \pm  0.94 $ \\
        $\eta = 1.0$ & $ 82.4 \pm 0.09 $ & $ \boldsymbol{88.7 \pm 0.47} $ & $ \boldsymbol{2.50} $ & $ 96.7 \pm 0.31 $ & $ \boldsymbol{98.0  \pm  0.24} $ \\
        $\eta = 2.0$ & $ 82.5 \pm 0.06 $ & $ 87.2 \pm 0.15 $ & $ 3.79 $ & $ 96.7 \pm 0.10 $ & $ 95.9  \pm  0.27 $ \\
        \hline
        \end{tabular}
        }
    \caption{Results on Bios with a \demonic~trained with domain adaptation. We report the mean $\pm$ standard deviation over 5 runs. * indicates the model without fairness consideration.}
    \label{tab:results_wfc_da}
\end{table}

\noindent
\customparagraph{\wfc~ results with cross-domain demonic}
Table \ref{tab:results_wfc_da} reports the results of the \wfc~ pipeline on the Bios data set when using domain adaptation during the \demonic~training. We note that thanks to the improvement of the accuracy of the \demonic~model, the fairness on the end-task is improved compared to both the pipeline without fairness consideration and the pipeline where the \demonic~model is trained without adaptation. With domain adaptation, the improved performance of the \demonic~model is reflected in the enhanced fairness. This experiment highlights the importance of an accurate \demonic~model and the advantages of considering domain adaptation when training it on data sets diverging from the end-task data set. 

\subsection{\wfc~ Architecture Components Investigation}
\label{sec:component_exploration}

In this section, we first investigate the hyperparameter $\beta$, its impact on the representation fairness (Leakage), and on fairness metrics related to the bounds from Lemma \ref{thm:boundedgroupfairness}. Next, we study the use of the representations from different layers in the two MLPs (classifier and demonic models). Finally, we explore the use of the predicted sensitive attributes instead of the hidden representations of the demonic model as previously done.

\subsubsection{Impact of the hyperparameter $\beta$}
\label{sec:beta_leakage}
\noindent
\customparagraph{Setting} In this experiment, we investigate the impact of the hyperparameter $\beta$ associated with the regularization term. Recall that our objective is the following
$$\arg \min \mathcal{L}(Y,h_y(Enc(X_{y}))) + \beta ~I_W(Z_y,Z_a),$$
where $\beta$ controls the impact of the Wasserstein Dependency Measure on the loss. 
We train the model over 5 seeds for different values of $\beta$. Specifically, $\beta \in \{0.1,~1,~2,~5,~10,~100\}$. \\

\begin{table}[t]
    \centering
    \resizebox{\textwidth}{!}{
    \begin{tabular}{l|cccc||cccc}
        \hline
         $\beta$ & \textbf{Acc. $\uparrow$} & \textbf{Fair. $\uparrow$} & \textbf{DTO $\downarrow$} & \textbf{Leak. $\downarrow$} & \textbf{Acc. $\uparrow$} & \textbf{Fair. $\uparrow$} & \textbf{DTO $\downarrow$} & \textbf{Leak. $\downarrow$} \\ \hline
         & \multicolumn{4}{c||}{\textbf{Bios}} & \multicolumn{4}{c}{\textbf{Moji}} \\ \hline
         0.1 & $ \boldsymbol{82.8 \pm 0.1} $ & $ 87.2 \pm 0.4 $ & $ 3.75 $ & $ 98.1 \pm 0.2 $ &  $ 50.4 \pm 0.7 $ & $ \boldsymbol{99.5 \pm 1.1} $ & $ 27.08 $ & $ 85.7 \pm 0.1 $ \\
        1.0 & $ 82.4 \pm 0.1 $ & $ \boldsymbol{89.0 \pm 0.3} $  & $ \boldsymbol{2.22} $ & $ 96.7 \pm 0.5 $ & $\boldsymbol{75.2 \pm 0.1}$  & $91.4 \pm 0.3$ &  $\boldsymbol{1.00}$  &  $86.9 \pm 0.2$  \\
        5.0 & $ 81.8 \pm 0.2 $ & $ 88.9 \pm 0.2 $  & $ 2.69 $ & $ 91.8 \pm 1.4 $ & $ 71.4 \pm 0.5 $ & $ 93.7 \pm 0.4 $ & $ 5.38 $ & $ \boldsymbol{81.1 \pm 0.5} $ \\ 
        10.0 & $ 81.6 \pm 0.2 $ & $ 88.6 \pm 0.2 $  & $ 3.06 $ & $ 86.1 \pm 0.8 $ & $ 70.1 \pm 0.6 $ & $ 92.7 \pm 0.4 $ & $ 6.21 $ & $ 82.5 \pm 0.8 $ \\ 
        100.0 & $ 81.2 \pm 0.4 $ & $ 87.9 \pm 0.4 $ & $ 3.84 $ & $ \boldsymbol{77.7 \pm 1.7} $ & $ 67.9 \pm 1.4 $ & $ 94.7 \pm 1.1 $ & $ 8.9 $ & $ 83.0 \pm 0.5 $ \\ \hline
    \end{tabular}}
    \caption{Study of the impact of $\beta$. We report the mean $\pm$ standard deviation over 5 runs.}
    \label{tab:beta_exp}
\end{table}

\noindent
\customparagraph{Discussion} First, we note in Table \ref{tab:beta_exp} that with a higher $\beta$, the Leakage decreases, meaning the sensitive attribute is harder to retrieve from the latent representations. Although we initially aim to improve the Fairness while maintaining the Accuracy of the model, our method can be used to improve the Leakage by increasing the value of $\beta$ in Equation \ref{eq:loss}. In other words, we give more importance to the Wasserstein regularization in the loss; as observed in Figure \ref{fig:wasserstein_beta_fairness_metrics} where increasing the importance of the regularization term allows having a lower $I_W(Z_y, Z_a)$.\footnote{Note that the values are computed exactly using the POT library \citep{flamary2021pot}.}
However, on both data sets, the Accuracy that we want to preserve decreases and the trade-off worsens as the Leakage gets better. In other words, reducing the Leakage makes it more challenging to retrieve sensitive attributes, but could result in unintended information loss needed for the classification task, affecting the performance. Ultimately, we want to enhance fairness while keeping a good performance, and this objective may not necessarily match with a strong Leakage improvement \citep{shen2022does}.
Indeed, \citet{lohaus2022two} show that training neural networks to satisfy fairness constraints (e.g., demographic parity) is often done by making the model more aware of the sensitive attributes, and that stronger fairness constraints make these attributes more recoverable from their internal states.
As such, combining privacy and fairness remains an open challenge.

Finally, note that on the Moji data set, the performance for $\beta=0.1$ is surprisingly low, this is due to the selection criterion used: the DTO. Indeed, when looking at the best results for this setting, we have an accuracy of $73\pm0.0$ for a fairness of $68.5\pm0.0$. This can be explained by the fact that the fairness regularization term is too low to improve fairness on this data set, then the results for the best accuracy are close to the \textit{CE}-baseline results (cf. Table \ref{tab:task1_moji}). However, at initialization with an inaccurate classifier, the fairness is very high, thus the optimal trade-off is obtained with these values.

In the next subsection, we investigate the relation of the Wasserstein Dependency Measure between the latent representations with the fairness metrics for different values of $\beta$.

    

\begin{figure}[h]
    \centering
    \begin{subfigure}{0.49\textwidth}
        \centering
        \includegraphics[width=\textwidth]{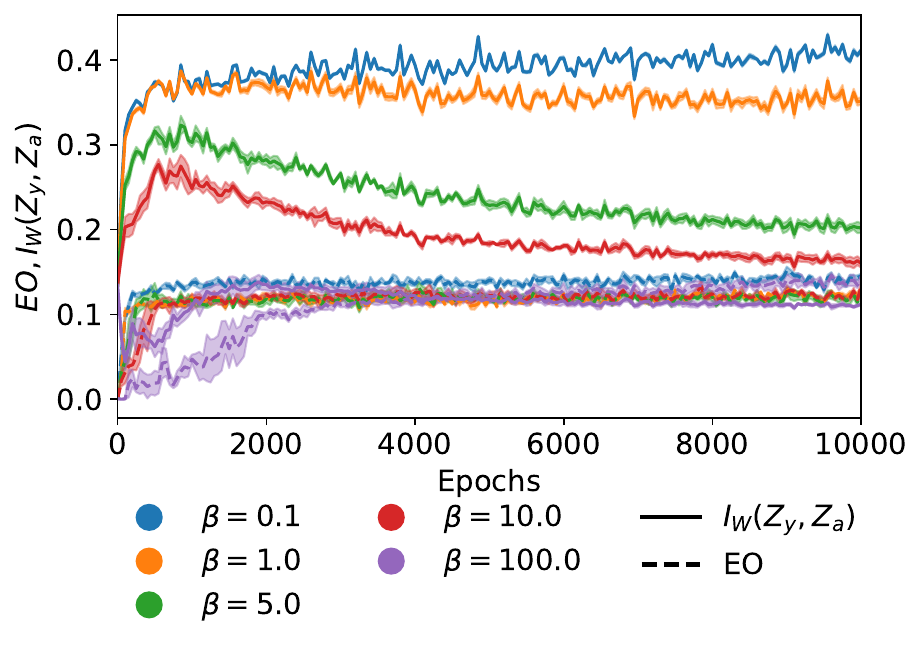}
        \caption{Bios - Equality of Opportunity}
        \label{fig:wasserstein_beta_fairness_metrics_a}
    \end{subfigure}
    \hfill
    \begin{subfigure}{0.49\textwidth}
        \centering
        \includegraphics[width=\textwidth]{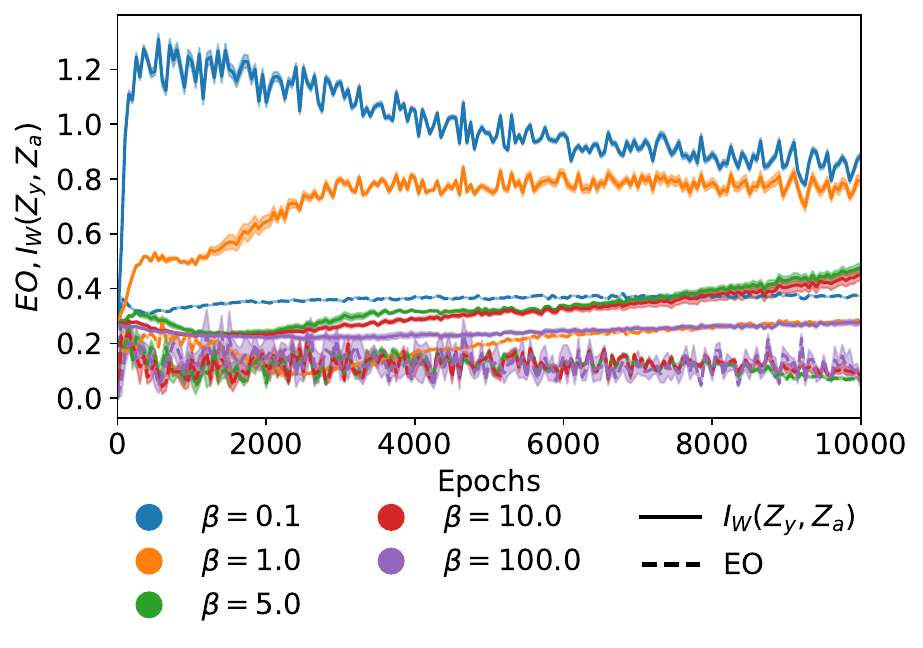}
        \caption{Moji - Equality of Opportunity}
        \label{fig:wasserstein_beta_fairness_metrics_b}
    \end{subfigure}

    \vspace{1em}

    \begin{subfigure}{0.49\textwidth}
        \centering
        \includegraphics[width=\textwidth]{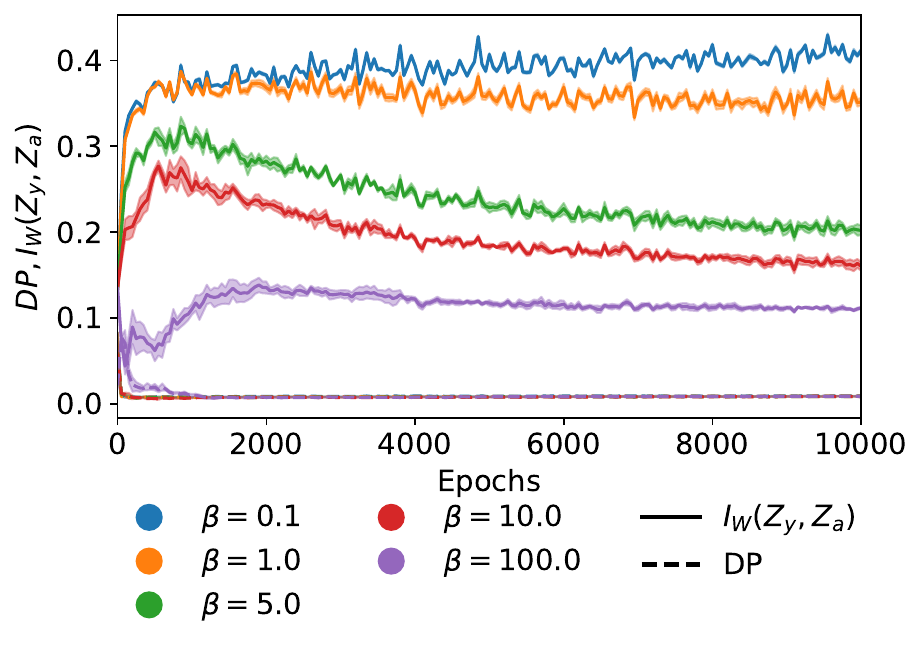}
        \caption{Bios - Demographic Parity}
        \label{fig:wasserstein_beta_fairness_metrics_c}
    \end{subfigure}
    \hfill
    \begin{subfigure}{0.49\textwidth}
        \centering
        \includegraphics[width=\textwidth]{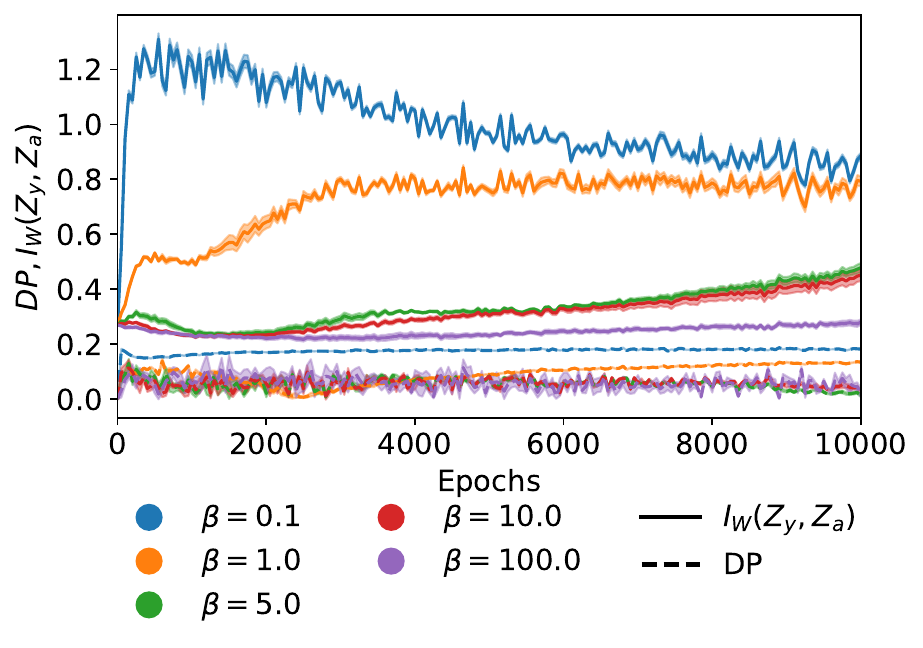}
        \caption{Moji - Demographic Parity}
        \label{fig:wasserstein_beta_fairness_metrics_d}
    \end{subfigure}
    \caption{$I_W(Z_y, Z_a)$ and averaged fairness metrics over classes across training epochs. The values are averaged over 5 runs. }%
    \label{fig:wasserstein_beta_fairness_metrics}%
\end{figure}

\subsubsection{Relation between fairness metrics and the regularization term}

In this section, we empirically show the validity of the bounds from Lemma \ref{thm:boundedgroupfairness} on two data sets, Bios and Moji. \\

\noindent
\customparagraph{Setting} We train the \wfc~ pipeline with different $\beta$ values as in Section \ref{sec:beta_leakage}, and report in Figure \ref{fig:wasserstein_beta_fairness_metrics}, the $I_W(Z_y, Z_a)$ on the training data, and the EO (\ref{fig:wasserstein_beta_fairness_metrics}a,b) and DP (\ref{fig:wasserstein_beta_fairness_metrics}c,d) on the test set for every training epoch. \\

\noindent
\customparagraph{Discussion} We note that the more the loss is constrained by the regularization term, the lower $I_W(Z_y, Z_a)$ is, as well as the fairness metrics. However, after a certain threshold value for $\beta$ ($5.0$ in the experiments), the fairness metrics converge. Finally, while in most cases $I_W(Z_y, Z_a)$ is greater than the considered metrics, as expected from Lemma \ref{thm:boundedgroupfairness} and Theorems \ref{theorem_in_domain_hat_approx} and \ref{theorem_latent_representations}, the contrary happens on a few epochs. 
This discrepancy with the results expected from the theoretical relation arises because we only plot $I_W(Z_y, Z_a)$ rather than the full right-hand term. 

\subsubsection{Use of representations from different layers}
\noindent
\customparagraph{Setting} In the previous experiments, following approaches presented in \citet{han2022fairlib}, the Wasserstein distance is approximated using the last hidden representations of the 3-layer MLP. In this section, we explore the use of representations from different layers of the MLPs. We compare this approach, on both data sets, with the use of the first hidden representations of the MLPs and with the output logits (before argmax), shown in Figure \ref{fig:archi_mlp}. For the latter, the Wasserstein distance is estimated between distributions of different dimensions. For example, for Bios, the \demonic~MLP predicts 2 labels while the classification MLP predicts 28 labels. \\


\begin{figure}[h]
    \centering
    \begin{minipage}{0.55\textwidth}
        \centering
            \resizebox{1.1\textwidth}{!}{
            \begin{tabular}{lcccc}
            \hline
            \textbf{Layer} & \textbf{Bal. Acc. $\uparrow$} & \textbf{Fairness $\uparrow$} & \textbf{DTO $\downarrow$} & \textbf{Leakage $\downarrow$} \\
            \hline
                & & \textbf{Bios} & & \\
                Last hidden & $ \textbf{82.4} \pm \textbf{0.1}^{*} $ & $ \textbf{89.0} \pm \textbf{0.3}^{*}$ & $ \textbf{2.06}^{*} $ & $ 96.5 \pm 0.5$ \\
                First hidden & $81.9 \pm 0.2$  &  $86.7 \pm 0.4$  & $ 4.29 $ & $96.5 \pm 0.6$  \\
                Output layer & $82.1 \pm 0.6$  &  $87.5 \pm 0.3$  & $ 3.49 $ & $\textbf{87.0} \pm \textbf{1.1}^{*}$  \\
                & & \textbf{Moji} & & \\
                Last hidden & $\textbf{75.2} \pm \textbf{0.1}^{*}$  & $\textbf{91.4} \pm \textbf{0.3}^{*}$ &  $\textbf{1.17}^{*}$  &  $86.9 \pm 0.2$  \\
                First hidden & $ 74.3 \pm 0.1 $ & $ 80.8 \pm 1.0 $ & $11.4$ & $ 85.6 \pm 0.6 $ \\
                Output layer & $ 73.5 \pm 0.0 $ & $ 70.2 \pm 0.2 $ & $21.9$ & $ \textbf{64.5} \pm \textbf{0.1}^{*} $ \\  
            \hline
            \end{tabular}}
        \captionof{table}{Comparison between the use of representations of different MLP layers to compute the Wasserstein.}
        \label{tab:diff_layer}
    \end{minipage}
    \hfill
    \begin{minipage}{0.35\textwidth}
        \centering
        \resizebox{\linewidth}{!}{
            \begin{tikzpicture}[
                node distance=1cm and 1cm,
                every node/.style={draw,circle,minimum size=1.2cm, align=center},
                input/.style={fill=blue!5},
                hidden/.style={fill=red!5},
                output/.style={fill=green!5}
            ]
            
                \node[input] (Input) at (0,6) {\textbf{Input} \\ ($\boldsymbol{|Emb|}$)};
                \node[hidden] (Hidden1) at (3,6) {\textbf{First} \\ \textbf{Hidden} \\ \textbf{(300)}};
                \node[hidden] (Hidden2) at (6,6) {\textbf{Last} \\ \textbf{Hidden} \\ \textbf{(300)}};
                \node[output] (Output) at (9,6) {\textbf{Output} \\ ($\boldsymbol{|C|}$)};
            
                \node[input] (I1) at (0,4) {};
                \node[input] (I2) at (0,2.5) {};
                \node [input] at (0,1) {$\vdots$}; 
                \node[input] (I3) at (0,-0.5) {};
                \node[input] (I4) at (0,-2) {};
            
                \node[hidden] (H1) at (3,3.5) {};
                \node[hidden] (H2) at (3,2) {};
                \node [hidden] (H3) at (3,0.5) {$\vdots$}; 
                \node[hidden] (H4) at (3,-1) {};
            
                \node[hidden] (H5) at (6,3.5) {};
                \node[hidden] (H6) at (6,2) {};
                \node [hidden] (H7) at (6,0.5) {$\vdots$}; 
                \node[hidden] (H8) at (6,-1) {};
            
                \node[output] (O1) at (9,3) {};
                \node [output] (O2) at (9,1.5) {$\vdots$}; 
                \node[output] (O3) at (9,0) {};
            
                \foreach \i in {I1,I2,I3,I4} {
                    \foreach \j in {H1,H2,H3,H4} {
                        \draw[-] (\i) -- (\j);
                    }
                }
                
                \foreach \i in {H1,H2,H3,H4} {
                    \foreach \j in {H5,H6,H7,H8} {
                        \draw[-] (\i) -- (\j);
                    }
                }
            
                \foreach \i in {H5,H6,H7,H8} {
                    \foreach \j in {O1,O2,O3} {
                        \draw[-] (\i) -- (\j);
                    }
                }
            
            \end{tikzpicture}
        }
        \caption{Representation of the MLP's layer. 
        }
        \label{fig:archi_mlp}
    \end{minipage}
\end{figure}

\noindent
\customparagraph{Discussion} On both data sets (Table \ref{tab:diff_layer}), accuracy is rather stable regardless of the layers used to compute the Wasserstein distance. Still, the best results are obtained using the last hidden representations. However, while we note a slight decrease in fairness on Bios when using representations from other layers, the decrease becomes much more significant on Moji. Thus, using the last hidden layer is the best strategy.

\subsubsection{Independence with predicted hard sensitive attributes}

\noindent
\customparagraph{Setting} To assess the impact of using the representation $Z_a$, we replace $Z_a$ with the sensitive attributes predicted by the \demonic~MLP, $\hat{A}$. We consider the setting using the embeddings from BERT, with the Bios and Moji data sets. Then, we 
obtain the following regularization term: $I_W(Z_y,\hat{A}) = W_1(p(Z_y,\hat{A}),p(Z_y)p(\hat{A}))$. Note that we do not encounter a problem with the non-differentiability for $\hat{A}$ (with the argmax operation as for $\hat{Y}$ as mentioned in Section \ref{sec:use_latent_rep}) since the \demonic~model is pre-trained. \\

\begin{table}[t]
    \centering
    \begin{tabular}{lcccc}
    \hline
    \textbf{Labels} & \textbf{Bal. Acc. $\uparrow$} & \textbf{Fairness $\uparrow$} & \textbf{DTO $\downarrow$} & \textbf{Leakage $\downarrow$} \\
    \hline
        & & \textbf{Bios} & & \\
        Representations & $ 82.4 \pm 0.1 $ & $ \textbf{89.0} \pm \textbf{0.3}^{*} $ & $ \textbf{2.06}^{*} $ & $ 96.5 \pm 0.5 $ \\
        Hard labels & $ \textbf{82.6} \pm \textbf{0.2} $ & $ 87.5 \pm 0.2 $ & $ 3.28 $ & $ \textbf{92.0} \pm \textbf{0.2}^{*} $ \\
        & & \textbf{Moji} & & \\
        Representations & $\textbf{75.2} \pm \textbf{0.1}^{*}$  & $\textbf{91.4} \pm \textbf{0.3}^{*}$ &  $\textbf{1.17}^{*}$  &  $86.9 \pm 0.2$  \\
        Hard labels &  $72.2 \pm 0.1$  &  $65.0 \pm 0.0$  & $27.3$ &  $\textbf{81.0} \pm \textbf{0.8}^{*}$  \\
    \hline
    \end{tabular}
    \caption{Comparison between the use of representations $Z_a$ and hard sensitive attributes to compute the Wasserstein distance.}
    \label{tab:s_hat}
\end{table}

\noindent
\customparagraph{Discussion} We report the results of this experiment in Table \ref{tab:s_hat}. When we replace $Z_a$ by the predicted $\hat{A}$ to compute the Wasserstein distance, we observe, on average, a slight improvement of the accuracy on Bios, and a slight decrease of the accuracy on Moji. However, while the decrease in Fairness is not significant for Bios, we observe a substantial drop for Moji. As a result, using $\hat{A}$ instead of $Z_a$ seems to have a neutral impact at best; this may also result, in some cases, in a reduction of both accuracy and fairness.

\section{Conclusion}\label{sec:conclusion}

We extend \wfc, a method enforcing fairness constraints using a pre-trained neural network on the sensitive attributes and Wasserstein regularization. We show that minimizing the Wasserstein Dependency Measure ($I_W$) improves fairness by reducing the statistical dependence between predictions and sensitive attributes, linking it to key metrics such as Demographic Parity and Equality of Opportunity.

Instead of directly optimizing $I_W$ between predictions and sensitive attributes, we apply it to the latent representations of two models: one predicting classification labels and the other sensitive attributes. We prove that this formulation provides an upper bound on the dependency measure between predictions and true sensitive attributes while ensuring computational feasibility. Specifically, the $I_W$ between latent representations upper-bounds the $I_W$ between predicted labels and sensitive attributes, which in turn upper-bounds the $I_W$ between predictions and true sensitive attributes. Our method does not require sensitive attribute annotations at both training and inference time. We obtain competitive results on the Bios data set and outperform baselines on fairness metrics while maintaining comparable accuracy on the Moji data set. The approach is also compatible with both encoder-based and decoder-based architectures.

We also extend our method to settings where sensitive attributes are unavailable, leveraging a domain adaptation approach to enable training under this constraint. We provide theoretical guarantees, inspired by domain adaptation results, to assess its generalization to other data sets. \\

\noindent
\customparagraph{Perspectives} Overall, this work could be extended in numerous ways.
The bound presented in Theorem \ref{theorem_cross_domain_hat_approx}  relies on a covariate shift assumption. However, in the literature on causal vs. anticausal learning \citep{ScholkopfJPSZM12}, many NLP tasks can be viewed as anticausal, where the target label (e.g., an occupation) generates the observed text (e.g., the description of the occupation) \citep{JinKNVKSS21}. In such cases, a prior probability shift assumption, where the label distribution changes across domains but the feature distribution conditioned on the label remains invariant, would arguably be a natural setting to consider. When the sensitive attribute itself is considered as the label (e.g., gender), as is the case for our demonic, one could also argue for an anticausal relation, but in a weaker sense: the sensitive attribute influences some parts of the text (e.g., pronouns, names) without fully determining its content. Studying fairness-aware transfer under such partial forms of anticausality would be an interesting direction for future work.

In Theorem 5, we discussed the use of an optimal softmax temperature to obtain a sharper bound. It would be interesting to investigate whether having such a given temperature is an artifact of our proof technique or whether it is a requirement. Furthermore, if it turns out to be necessary, a possible avenue of research would be to draw connections with other research domains that rely on temperature scaling. Calibration is a prime example of such a field where it is a popular baseline that has been shown to perform well in different settings \citep{park2020calibrated,wang2020transferable,chen2023transfer,guo2017calibration}. Hence, it would be interesting to study in which cases the temperature derived in Theorem 5 would positively or negatively impact the level of calibration of the model.

Finally, although we did not explore this direction, the approach could be extended beyond text classification to tasks such as regression or unsupervised learning, or to other types of data such as images.

\section{Limitations}

The proposed approach is flexible and can handle various types of sensitive attributes. However, due to the lack of available data sets, we were unable to evaluate its performance on continuous sensitive attributes, such as age. Additionally, while gender can be represented as an $n$-ary variable, our experiments were limited to a binary classification (men vs. women) due to data availability.


Finally, our experiments demonstrate the effectiveness of our approach in transferring sensitive attributes to improve fairness. However, our theoretical results indicate that the success of this transfer depends on its quality; a poor transfer could, in theory, lead to a decrease in fairness.


\acks{

This work is funded by the French National Research Agency (ANR) in the context of the grant ANR-21-CE23-0026 (Project DIKÉ).
Michael Perrot is supported by the ANR through the grant ANR-23-CE23-0011-01 (Project FaCTor). Charlotte Laclau is supported by the ANR through the grant ANR-23-CE23-0026 (Project ReFAIR) and through the PEPR IA FOUNDRY (ANR-23-PEIA-0003).
Our experiments use the previously mentioned
Fairlib framework. We would like to express our
gratitude to Xudong Han for his availability and
assistance in using it.}


\appendix

\section{Wasserstein Distance}
\label{wasserstein_for_fairness}

Finding correspondences between two sets of points is a longstanding issue in machine learning. The optimal transport (OT) \citep{monge_81} problem offers an efficient solution to this issue by calculating an optimal one-to-one transport map between the two sets, taking into account the geometrical proximity of the points. 
Let $\hat{\mu}_0$ and $\hat{\mu}_1$ be measures supported on the point sets $X_0 = \{x_0^{(i)} \in \mathbb{R}^d\}_{i=1}^{N_0}$ and $X_1 = \{x_1^{(j)} \in \mathbb{R}^d\}_{i=1}^{N_1}$, respectively. We consider the Monge-Kantorovich formulation of the original OT problem \citep{kantorovich}, where the goal is to find a coupling $\gamma$ defined as a joint probability distribution over $X_0 \times X_1$ with marginals $\hat{\mu}_0$ and $\hat{\mu}_1$. This amounts to  minimizing the cost of transport w.r.t. some metric $l_p = \| \cdot\|_p:X_0 \times X_1 \rightarrow \mathbb{R}^+$, the $l_p$-norm.
This problem admits a unique solution $\gamma^*$ and defines a metric on the space of probability measures called the Wasserstein distance (also known as the Earth-Mover Distance) as follows

$$W_1(\hat{\mu}_0, \hat{\mu}_1) = \min_{\gamma \in \Pi(\hat{\mu_0}; \hat{\mu}_1)}\langle M, \gamma\rangle_F,$$
where $\langle \cdot \text{,} \cdot \rangle_F$ is the Frobenius dot product, $M$ is a dissimilarity matrix, i.e., $M_{ij} = l(x_0^{(i)},x_1^{(j)})$, defining the cost of associating $x_0^{(i)}$ with $x_1^{(j)}$ and $\Pi(\hat{\mu}_0, \hat{\mu}_1) = \lbrace \gamma \in \mathbb{R}^{N_0 \times N_1}_+ \vert \gamma \bm{1} = \hat{\mu}_0, \gamma^T \bm{1} = \hat{\mu}_1\rbrace$ is a set of doubly stochastic matrices.


In the following, we will rely on the following technical lemma on the Wasserstein distance between discrete distributions.

\begin{lemma}\label{lem:blocktransportcost}
Let $U\sim p(U)$ and $V \sim p(V)$ be two discrete random variables respectively taking values in $u_1,\ldots,u_k$ and $v_1,\ldots,v_k$. Assume that $\normp{u_i - v_j} = \left\{\begin{array}{cl}
0 & \text{if $i=j$}  \\
\sqrt[p]{2} & \text{otherwise}
\end{array}\right.$, then, we have that
\begin{align*}
    W_1(p(U),p(V)) ={}&  \frac{\sqrt[p]{2}}{2}\sum_{i=1}^k \left|\mathbb{P}(u_i) - \mathbb{P}(v_i)\right|
\end{align*}
\end{lemma}

\begin{proof}
\nolinenumbers
From \citet[Theorem~4]{gibbs2002}, we have that
\begin{align*}
     \min_{u\neq v} \normp{u-v} TV(p(U),p(V)) \leq W_1(p(U), p(V)) \leq \max_{u,v} \normp{u-v} TV(p(U),p(V)),
\end{align*}
where $TV(p(U),p(V)) = \frac{1}{2}\sum_{i=1}^k \left|\mathbb{P}(u_i) - \mathbb{P}(v_i)\right|$ is the total variation. Noticing that, in our case, $\min_{u\neq v} \normp{u-v} = \max_{u,v} \normp{u-v} = \sqrt[p]{2}$ concludes thee proof.

\end{proof}

\begin{lemma}\label{lem:dependencytransportcost}
Let $U\sim p(U)$, $V\sim p(V)$, and $W \sim p(W)$ be discrete random variables taking values in $\mathcal{U}$, $\mathcal{V}$, and $\mathcal{W}$ respectively and such that $\normp{u - u'} = \normp{v - v'} =\normp{w - w'} = \left\{\begin{array}{cl}
0 & \text{if $u=u'$, $v=v'$ or $w=w'$}  \\
\sqrt[p]{2} & \text{otherwise}
\end{array}\right.$, then, we have that
\begin{align*}
    W_1(p(U,W),p(U)p(W))
    ={}& \sum_w W_1(p(U|W=w),p(U)) \mathbb{P}(W=w) \\
    W_1(p(U,W),p(V,W))
    ={}& \sum_w W_1(p(U|W=w),p(V|W=w)) \mathbb{P}(W=w) \\
    W_1(p(U)p(W),p(V)p(W))
    ={}& \sum_w W_1(p(U),p(V)) \mathbb{P}(W=w)
\end{align*}
\end{lemma}
\begin{proof}
\nolinenumbers
    The cost matrix associated with $W_1(p(U,W),p(U)p(W))$ is of size $|\mathcal{U}||\mathcal{W}| \times |\mathcal{U}||\mathcal{W}|$. Assuming that we order the pairs $(u,w)$ by varying the values of $u$ first, that is $(u_1,w_1),$ $(u_2,w_1),\ldots$, the cost matrix contains blocks of size $|\mathcal{U}| \times |\mathcal{U}|$. The diagonal blocks have value $\sqrt[p]{2}(\mathbbm{1}_{|\mathcal{U}|\times |\mathcal{U}|}-\mathbb{I}_{|\mathcal{U}|\times |\mathcal{U}|})$ where $\mathbb{I}_{|\mathcal{U}|\times |\mathcal{U}|}$ is the identity matrix of size $|\mathcal{U}| \times |\mathcal{U}|$ and $\mathbbm{1}_{|\mathcal{U}|\times |\mathcal{U}|}$ is a matrix of ones. The off diagonal blocks have value $\sqrt[p]{2}\mathbb{I}_{|\mathcal{U}|\times |\mathcal{U}|} + \sqrt[p]{4}(\mathbbm{1}_{|\mathcal{U}|\times |\mathcal{U}|}-\mathbb{I}_{|\mathcal{U}|\times |\mathcal{U}|})$.

    We have that $\forall w \in \mathcal{W}, \sum_u \mathbb{P}(U=u,W=w) = \mathbb{P}(W=w) = \sum_u \mathbb{P}(U=u)\mathbb{P}(W=w)$ which means that we can consider each diagonal block independently when computing $W_1(p(U,W),p(U)p(W))$, that is compute $\forall w, W_1(p(U|W=w),p(U))$ and then normalize the transport cost by $\mathbb{P}(W=w)$. This will be the optimal cost since the mass that is not transported with a cost of $0$ will be transported with a cost of $\sqrt[p]{2}$, which is the smallest possible cost different from $0$. This concludes the proof of the first equality.
    The proofs of the second and third equality follow using the same arguments.
\end{proof}

\section{Connection with Group Fairness}
\label{app:group_fairness}
The following lemma shows that minimizing the Wasserstein dependency measure is a sound way to improve either demographic parity or equality of opportunity. \\

\noindent
\textbf{Lemma \ref{thm:boundedgroupfairness} (Group fairness and Wasserstein dependency measure.)}
\textit{Let $I_W$ be the Wasserstein dependency measure, and $A$, $Y$, $\hat{Y}$ be random variables corresponding to the sensitive attribute, the true label, and the predicted label, respectively. We have that
\begin{align*}
    I_W(\hat{Y},A) ={}& \frac{\sqrt[p]{2}}{2} \sum_{a \in \mathcal{A}} \mathbb{P}(A=a) \sum_{y\in\mathcal{Y}} \textbf{DP}_{a,y} \;\text{,} \\
    I_W((\hat{Y}=Y)|Y=y,A|Y=y) ={}& \sqrt[p]{2}\sum_{a \in \mathcal{A}} \mathbb{P}(A=a|Y=y) \textbf{EO}_{a,y} \;\text{.}
\end{align*}}
\begin{proof}
\nolinenumbers
Let $\hat{Y}$ and $A$ be the two random variables corresponding to the predicted label and sensitive attribute, respectively. Recall that these random variables are encoded using a one-hot vector, that is $\normp{y_i-y_j} = \left\{\begin{array}{cl}
   0  & \text{if $i=j$} \\
   \sqrt[p]{2}  & \text{otherwise}
\end{array}\right.$ and $\normp{a_i-a_j} = \left\{\begin{array}{cl}
   0  & \text{if $i=j$} \\
   \sqrt[p]{2}  & \text{otherwise}
\end{array}\right.$. Then, by successively applying Lemma~\ref{lem:dependencytransportcost} and Lemma~\ref{lem:blocktransportcost}, we have that
\begin{align*}
    I_W(\hat{Y},A) :={}& W_1(p(\hat{Y}, A), p(\hat{Y})p(A)) \\
    ={}& \sum_{a \in \mathcal{A}} W_1(p(\hat{Y} | A=a), p(\hat{Y}))\mathbb{P}(A=a) \\
    ={}& \sum_{a \in \mathcal{A}} \mathbb{P}(A=a) \frac{\sqrt[p]{2}}{2}\sum_{y\in\mathcal{Y}} \left|\mathbb{P}(\hat{Y}=y|A=a) - \mathbb{P}(\hat{Y}=y)\right| \\
    ={}& \frac{\sqrt[p]{2}}{2} \sum_{a \in \mathcal{A}} \mathbb{P}(A=a) \sum_{y\in\mathcal{Y}} \left|\mathbb{P}(\hat{Y}=y|A=a) - \mathbb{P}(\hat{Y}=y)\right|
\end{align*}
Noticing that $\left|\mathbb{P}(\hat{Y}=y|A=a) - \mathbb{P}(\hat{Y}=y)\right|$ is the demographic parity for group $a$ and label $y$ concludes the proof of the first statement.
Similarly, notice that given a label $y \in \mathcal{Y}$
\begin{small}
\begin{align*}
    I_W((\hat{Y}=Y)|Y=y,A|Y&=y) :={} W_1(p((\hat{Y}=Y), A|Y=y), p((\hat{Y}=Y)|Y=y)p(A|Y=y)) \\
    ={}& \sum_{a \in \mathcal{A}} W_1(p((\hat{Y}=Y) | A=a,Y=y), p((\hat{Y}=Y)|Y=y))\mathbb{P}(A=a|Y=y) \\
    ={}& \sum_{a \in \mathcal{A}} \mathbb{P}(A=a|Y=y) \frac{\sqrt[p]{2}}{2}\left(\left|\mathbb{P}(\hat{Y}=Y|A=a,Y=y) - \mathbb{P}(\hat{Y}=Y|Y=y)\right|\right.\\
    &+ \left. \left|\mathbb{P}(\hat{Y}\neq Y|A=a,Y=y) - \mathbb{P}(\hat{Y}\neq Y|Y=y)\right|\right) \\
    ={}& \sqrt[p]{2}\sum_{a \in \mathcal{A}} \mathbb{P}(A=a|Y=y)\left|\mathbb{P}(\hat{Y}=Y|A=a,Y=y) - \mathbb{P}(\hat{Y}=Y|Y=y)\right|
\end{align*}
\end{small}
Noticing that $\left|\mathbb{P}(\hat{Y}=Y|A=a,Y=y) - \mathbb{P}(\hat{Y}=Y|Y=y)\right|$ is the Equality of opportunity for group $a$ and label $y$ concludes the proof of the second statement.
\end{proof}

\section{Bounding the $I_W(\hat{Y},A)$ by the error rate}
\label{app:bound_prediction}
In this section, we provide the details of the proof of Lemma \ref{lemma_hat_approx} leading to Theorems \ref{theorem_in_domain_hat_approx} and \ref{theorem_cross_domain_hat_approx}. \\

\noindent
\textbf{Lemma \ref{lemma_hat_approx}}
\textit{Let $\hat{Y},\hat{A},A$ be random variables that correspond to the predicted label, predicted sensitive attribute, and true sensitive attribute, respectively. Then, we have that
$$I_W(\hat{Y},A) \leq I_W(\hat{Y},\hat{A}) + 2 \sqrt[p]{2} \Pb(A\neq\hat{A})$$}

\begin{proof}
Let $\hat{Y}, \hat{A}$ and $A$ be the random variables corresponding to the predicted label, predicted sensitive attribute, and true sensitive attribute, respectively. The Wasserstein Dependency Measure \citep{ozair2019wasserstein} is
\begin{align*}
I_W(\hat{Y}, A) = W_1(p(\hat{Y}, A), p(\hat{Y})p(A)) \;\text{.}
\end{align*}

\noindent
The $W_1$-metric can be shown to be a proper metric when the compared distributions have the same overall mass \citep{rubner2000}. Therefore, it satisfies the triangle inequality
\begin{align*}
    I_W(\hat{Y}, A) :={}& W_1(p(\hat{Y}, A), p(\hat{Y})p(A)) \\
    \leq{}& W_1(p(\hat{Y},A), p(\hat{Y},\hat{A})) \\
    &+ W_1(p(\hat{Y}, \hat{A}), p(\hat{Y})p(\hat{A})) \\ 
    &+ W_1(p(\hat{Y})p(\hat{A}), p(\hat{Y})p(A)), 
\end{align*}
with $W_1(p(\hat{Y}, \hat{A}), p(\hat{Y})p(\hat{A})) = I_W(\hat{Y},\hat{A})$.

Recall that $\hat{Y}, \hat{A}$ and $A$ are encoded using a one hot vector, that is $\normp{y_i-y_j} = \left\{\begin{array}{cl}
   0  & \text{if $i=j$} \\
   \sqrt[p]{2}  & \text{otherwise}
\end{array}\right.$ and $\normp{a_i-a_j} = \left\{\begin{array}{cl}
   0  & \text{if $i=j$} \\
   \sqrt[p]{2}  & \text{otherwise}
\end{array}\right.$. Then, by successively applying Lemma~\ref{lem:dependencytransportcost} and Lemma~\ref{lem:blocktransportcost}, we have that
\begin{align}
    W_1(p(\hat{Y}, A), p(\hat{Y},\hat{A})) ={}& \sum_{y \in \mathcal{Y}} W_1(p(A | \hat{Y}=y),p(\hat{A} | \hat{Y}=y))\mathbb{P}(\hat{Y}=y) \nonumber \\
    ={}& \sum_{y \in \mathcal{Y}} \mathbb{P}(\hat{Y}=y)\frac{\sqrt[p]{2}}{2}\sum_{a\in \mathcal{A}} \left|\mathbb{P}(A=a|\hat{Y}=y) - \mathbb{P}(\hat{A}=a|\hat{Y}=y)\right| \label{eq:dev_lem_8}
\end{align}
By the law of total probability and the union bound for disjoint events, we have that
\begin{align*}
    \sum_{a\in \mathcal{A}}& \left|\mathbb{P}(A=a|\hat{Y}=y) - \mathbb{P}(\hat{A}=a|\hat{Y}=y)\right| \\
    ={}& \sum_{a\in \mathcal{A}} \left|\mathbb{P}(A=a, \hat{A} = a|\hat{Y}=y) + \mathbb{P}(A=a, \hat{A}\neq a|\hat{Y}=y)\right.\\
    & \left.- \mathbb{P}(\hat{A}=a, A=a|\hat{Y}=y) -  \mathbb{P}(\hat{A}=a, A\neq a|\hat{Y}=y)\right| \\
    ={}& \sum_{a\in \mathcal{A}} \left|\mathbb{P}(A=a, \hat{A}\neq a|\hat{Y}=y) -  \mathbb{P}(\hat{A}=a, A\neq a|\hat{Y}=y)\right| \\
    \leq{}& \sum_{a\in \mathcal{A}} \mathbb{P}(A=a, \hat{A}\neq a|\hat{Y}=y) +  \mathbb{P}(\hat{A}=a, A\neq a|\hat{Y}=y) \\
    ={}& \sum_{a\in \mathcal{A}} \mathbb{P}(A=a, \hat{A}\neq a|\hat{Y}=y) +  \mathbb{P}(\hat{A}=a, A\neq a|\hat{Y}=y) \\
    ={}& \mathbb{P}(\bigcup_{a\in \mathcal{A}} A=a, \hat{A}\neq a|\hat{Y}=y) +  \mathbb{P}(\bigcup_{a\in \mathcal{A}}\hat{A}=a, A\neq a|\hat{Y}=y) \\
    ={}& 2\mathbb{P}( A\neq \hat{A}|\hat{Y}=y)
\end{align*}
Plugging this result in Equation~\ref{eq:dev_lem_8}, we obtain
\begin{align*}
    W_1(p(\hat{Y}, A), p(\hat{Y},\hat{A})) ={}& \sum_{y \in \mathcal{Y}} \mathbb{P}(\hat{Y}=y)\sqrt[p]{2}\mathbb{P}( A\neq \hat{A}|\hat{Y}=y) \\
    ={}&\sqrt[p]{2}\mathbb{P}( A\neq \hat{A})
\end{align*}

Using similar arguments, we obtain that
\begin{align*}
    W_1(p(\hat{Y})p(A), p(\hat{Y})p(\hat{A})) ={}& \sum_{y \in \mathcal{Y}} \mathbb{P}(\hat{Y}=y)\frac{\sqrt[p]{2}}{2}\sum_{a\in \mathcal{A}} \left|\mathbb{P}(A=a) - \mathbb{P}(\hat{A}=a)\right| \\
    ={}&\sqrt[p]{2}\mathbb{P}( A\neq \hat{A}).
\end{align*}
This concludes the proof of this lemma.

\end{proof}
Built on this first result, we consider two scenarios to bound the error rate: either we pre-trained the \demonic~model on the data of the classification task (\textbf{in-domain}) or as a \da~ problem, on different data with shared sensitive attributes (e.g., gender, ethnicity, etc.) (\textbf{cross-domain)}.

\subsection{In-domain bound of the error rate for binary sensitive attributes}
\label{app:bound_in_domain}
\textbf{Theorem \ref{theorem_in_domain_hat_approx}}
\textit{
Let $\hat{A}, A \in \{0, 1\}$, and $\mathcal{H}$ be a hypothesis space of $VC$-dimension $d$. Let $\normp{\cdot}$ be the ground metric for the Wasserstein 1-distance. Assume that we have access to a training set of $m$ i.i.d. examples. Then, with probability at least $1-\delta$, we have $\forall h \in \mathcal{H}$
\begin{align*}
    I_W(\hat{Y}, A) \leq I_W(\hat{Y}, \hat{A}) + 2 \sqrt[p]{2} \left(\hat{\varepsilon} + \sqrt{\frac{4}{m} \left(dlog\frac{2em}{d}+log\frac{4}{\delta}\right)}\right)
\end{align*}
with $e$, the base of the natural logarithm and $\hat{\varepsilon}$ the empirical risk of the demonic model. 
} \\

\begin{proof}
From Lemma \ref{lemma_hat_approx}, we derive the following
\begin{align*}
        I_W(\hat{Y}, A) &\leq I_W(\hat{Y}, \hat{A}) + 2 \sqrt[p]{2}\Pb(A\neq\hat{A}) \\
        &= I_W(\hat{Y}, \hat{A}) + 2 \sqrt[p]{2} \varepsilon \\
\end{align*}
We apply the Vapnik-Chervonenkis theory \citep{vapnik1998statistical} to bound the true error $\varepsilon$ of the \demonic~model by its empirical risk $\hat{\varepsilon}$.
Let  $h_a$ be a fixed classification function from $Z_a$ to $A$ and $\mathcal{H}$ be a hypothesis space of
$VC$-dimension $d$. 
Therefore, if the training set is of size $m$ .i.i.d. samples, with probability at least $1-\delta$, we have for every $h \in \mathcal{H}$
\begin{align*}
    I_W(\hat{Y}, A) \leq I_W(\hat{Y}, \hat{A}) + 2 \sqrt[p]{2}\left(\hat{\varepsilon} + \sqrt{\frac{4}{m} \left(dlog\frac{2em}{d}+log\frac{4}{\delta}\right)}\right)
\end{align*}
and $e$ is the base of the natural logarithm.

\end{proof}

\section{Bounding $I_W(\hat{Y},\hat{A})$ by $I_W(Z_y,Z_a)$}
\label{app:bound_rep}

In this section, we present the proof for Theorem \ref{theorem_latent_representations} recalled below. \\

\noindent
\textbf{Theorem \ref{theorem_latent_representations}}
\textit{Let $\hat{Y},\hat{A}$ be random variables that correspond to the predicted label and predicted sensitive attribute, respectively. Assume that $h_y = \sigma_{\lambda}(f(Z_y))$ and $h_a = \sigma_{\lambda}(g(Z_a))$ where $\sigma_{\lambda}$ is the softmax function with temperature $\lambda$, $f$ and $g$ are both $L$-lipschitz with respect to the $p$-norm, and $Z_y$ and $Z_a$ are latent representations of the examples. Let $\normp{\cdot}$ be the ground metric for the Wasserstein 1-distance. For a given example $x$ with predicted label $\hat{y}$ and predicted sensitive attribute $\hat{a}$, let $\xi_y(x) = f(Z_y)_{\hat{y}} - \max_{y' \neq \hat{y}}f(Z_y)_{y'}$ and $\xi_a(x) = g(Z_a)_{\hat{a}} - \max_{a' \neq \hat{a}}g(Z_a)_{a'}$ be positive margins. Let $\delta = 1-\mathbb{P}(\xi_y(X)\geq \xi, \xi_a(X)\geq \xi)$ with $\xi > 0$. Let $\alpha=\sqrtp{2} \normp{\binom{|\mathcal{Y}|}{|\mathcal{A}|} - 1}(1-\delta)$ and $\iota=L(\abs{\mathcal{Y}}+\abs{\mathcal{A}})^{\abs{\frac{1}{2}-\frac{1}{p}}}$. Then, setting $\lambda = \frac{1}{\xi}\log\left(\frac{2\xi\alpha}{\iota I_W(Z_y, Z_a)}\right)$, we have that
\begin{align*}
        I_W(\hat{Y},\hat{A}) \leq{}& \min\left( \alpha, 2I_W(Z_y, Z_a)\frac{\iota}{\xi} \left[1+\log \left( \max\left(4, \frac{2\xi\alpha}{\iota I_W(Z_y, Z_a)}\right) - 1\right)\right] \right)\\
        &+ \sqrtp{2} \normp{\binom{|\mathcal{Y}|}{|\mathcal{A}|}-1} \delta.
\end{align*}
}

\begin{proof}
    Since, in our case, the Wasserstein distance is a proper metric, we have that
    \begin{align}
        I_W(\hat{Y},\hat{A}) ={}& W_1(p(\hat{Y}, \hat{A}), p(\hat{Y})p(\hat{A})), \nonumber \\
        \leq{}& W_1(p(\hat{Y}, \hat{A}), p(\sigma_\lambda(f(Z_y)), \sigma_\lambda(g(Z_a)))) \nonumber \\
        &+ W_1(p(\sigma_\lambda(f(Z_y)), \sigma_\lambda(g(Z_a))), p(\sigma_\lambda(f(Z_y)))p(\sigma_\lambda(g(Z_a)))) \nonumber \\
        &+ W_1(p(\sigma_\lambda(f(Z_y)))p(\sigma_\lambda(g(Z_a))), p(\hat{Y})p(\hat{A})). \label{eq:three_terms_bound}
    \end{align}
    We will first bound each term independently and then show that we can choose the softmax temperature, $\lambda$, in order to minimize the right-hand side of the bound. \\

    \noindent
    \customparagraph{Bounding $W_1(p(\sigma_\lambda(f(Z_y)), \sigma_\lambda(g(Z_a))), p(\sigma_\lambda(f(Z_y)))p(\sigma_\lambda(g(Z_a))))$.} Given $\gamma \in \Gamma$ a coupling between the two distributions, the second term can be bounded as
    \begin{align*}
        W_1(p(\sigma_\lambda(f(Z_y)),& \sigma_\lambda(g(Z_a))), p(\sigma_\lambda(f(Z_y)))p(\sigma_\lambda(g(Z_a)))), \\
        =& \inf_{\gamma} \mathbb{E}_{(z_y,z_a,z'_y,z'_a) \sim \gamma} \normp{(\sigma_\lambda(f(z_y)),\sigma_\lambda(g(z_a))) - (\sigma_\lambda(f(z'_y)),\sigma_\lambda(g(z'_a)))}, \\
        \leq{}& \inf_{\gamma} \mathbb{E}_{(z_y,z_a,z'_y,z'_a) \sim \gamma} L\lambda(\abs{\mathcal{Y}}+\abs{\mathcal{A}})^{\abs{\frac{1}{2}-\frac{1}{p}}}\normp{(z_y,z_a) - (z'_y,z'_a)}, \\
        \leq{}& L\lambda(\abs{\mathcal{Y}}+\abs{\mathcal{A}})^{\abs{\frac{1}{2}-\frac{1}{p}}} W_1(p(Z_y, Z_a), p(Z_y)p(Z_a)), \\
        ={}& L\lambda(\abs{\mathcal{Y}}+\abs{\mathcal{A}})^{\abs{\frac{1}{2}-\frac{1}{p}}} I_W(Z_y,Z_a).
    \end{align*}
    where the first inequality comes from the $\lambda$-lipschitzness of the softmax function $\ell_2$-norm \citep{gao2017properties} and equivalence of norms properties. \\

    \noindent
    \customparagraph{Bounding $W_1(p(\hat{Y}, \hat{A}), p(\sigma_\lambda(f(Z_y)), \sigma_\lambda(g(Z_a))))$ and $W_1(p(\sigma_\lambda(f(Z_y)))p(\sigma_\lambda(g(Z_a))), p(\hat{Y})p(\hat{A}))$.}
    Let the softmax function $\sigma_{\lambda}(f(z)) = \frac{e^{\lambda f(z)}}{ \left\|e^{\lambda f(z)}\right\|_1}$ for $z$ a vector representation of an example $x$ and $\lambda \geq 0$ the temperature. Then, we have that
    \begin{align*}
    W_1(p(\hat{Y}, \hat{A}), p(\sigma_{\lambda}(f(Z_y)), \sigma_{\lambda}(g(Z_a))) 
        ={}& W_1(p(\hat{Y})p(\hat{A}), p(\sigma_{\lambda}(f(Z_y)))p(\sigma_{\lambda}(f(Z_a))) \\
        ={}& \mathop{\mathbb{E}}c(X, X).
    \end{align*}
    Indeed, for an example $x$ represented as $z_y$, $z_a$ and with predictions $\hat{y}$ and $\hat{a}$ and an example $x'$ represented as $z'_y$, $z'_a$ and with predictions $\hat{y}$ and $\hat{a}$ the cost matrix $c$ is such that:
    $c(x,x') = \left\|(\hat{y},\hat{a})^\top - \left(\frac{e^{\lambda f(z'_y)}}{\|e^{\lambda f(z'_y)}\|_1}, \frac{e^{\lambda f(z'_a)}}{\|e^{\lambda f(z'_a)}\|_1}\right)^\top \right\|_p$. Thus, the minimal cost is achieved when each example is mapped onto itself since the predictions are obtained by taking the labels and sensitive attributes predicted as being most likely.
    We then have that \\
    
    \resizebox{0.96\textwidth}{!}{$
    \begin{aligned}
        c(x,x) &= \left[\left(1 - \frac{e^{\lambda f(z_y)_{\hat{y}}}}{\|e^{\lambda f(z_y)}\|_1}\right)^p
        + \left(\frac{\sum\limits_{y' \neq \hat{y}}e^{\lambda f(z_y)_{y'}}}{\|e^{\lambda f(z_y)}\|_1}\right)^p
        + \left(1 - \frac{e^{\lambda g(z_a)_{\hat{a}}}}{\|e^{\lambda g(z_a)}\|_1}\right)^p
        + \left(\frac{\sum\limits_{a' \neq \hat{a}}e^{\lambda g(z_a)_{a'}}}{\|e^{\lambda g(z_a)}\|_1}\right)^p\right]^{\frac{1}{p}}, \\
        &= \left[ 2~ \left(\frac{\sum\limits_{y' \neq \hat{y}}e^{\lambda f(z_y)_{y'}}}{\|e^{\lambda f(z_y)}\|_1}\right)^p 
        + 2~ \left(\frac{\sum\limits_{a' \neq \hat{a}}e^{\lambda g(z_a)_{a'}}}{\|e^{\lambda g(z_a)}\|_1}\right)^p\right]^{\frac{1}{p}}.
    \end{aligned}
    $} \\
    
    For a given example $x$ with predicted label $\hat{y}$ and predicted sensitive attribute $\hat{a}$, let $\xi_y(x) = f(z_y)_{\hat{y}} - \max_{y' \neq \hat{y}}f(z_y)_{y'}$ and $\xi_a(x) = g(z_a)_{\hat{a}} - \max_{a' \neq \hat{a}}g(z_a)_{a'}$ be positive margins. We note $m_y = f(z_y)_{\hat{y}}$ and $m_a = f(z_a)_{\hat{a}}$.
    Then, we have that
    \begin{align*}
        c(x,x) & \leq \left[2 \left(\frac{(|\mathcal{Y}|-1)e^{\lambda(m_y-\xi_y)}}{e^{\lambda m_y + e^{\lambda(m_y-\xi_y)}}}\right)^p
        + 2 \left(\frac{(|\mathcal{A}|-1)e^{\lambda(m_a-\xi_a)}}{e^{\lambda m_a + e^{\lambda(m_a-\xi_a)}}}\right)^p\right]^{\frac{1}{p}}, \\
        & \leq \left[2 \left(\frac{(|\mathcal{Y}|-1)e^{\lambda m_y}}{e^{\lambda \xi_y}e^{\lambda m_y} + e^{\lambda m_y}}\right)^p
       + 2 \left(\frac{(|\mathcal{A}|-1)e^{\lambda m_a}}{e^{\lambda \xi_a}e^{\lambda m_a} + e^{\lambda m_a}}\right)^p\right]^{\frac{1}{p}}, \\
       & \leq \left[2 \left(\frac{|\mathcal{Y}| - 1}{e^{\lambda \xi_y} + 1}\right)^p + 2 \left(\frac{|\mathcal{A}| - 1}{e^{\lambda \xi_a} + 1}\right)^p\right]^{\frac{1}{p}}.
    \end{align*}

    Let $\delta = 1-\mathbb{P}(\xi_y(X)\geq \xi, \xi_a(X)\geq \xi)$ with $\xi > 0$, then we have that
    \begin{align*}
        \mathop{\mathbb{E}} c(X, X) ={}& \mathop{\mathbb{E}}\left[ c(X, X) | \xi_y(x)\geq \xi, \xi_a(x)\geq \xi \right](1-\delta) + \mathop{\mathbb{E}}\left[ c(X, X) | \overline{\xi_y(x)\geq \xi, \xi_a(x)\geq \xi} \right]\delta, \\
        \leq{}& \left[2 \frac{\left(|\mathcal{Y}| - 1\right)^p + \left(|\mathcal{A}| - 1\right)^p}{\left(e^{\lambda \xi} + 1 \right)^p} \right]^{\frac{1}{p}} (1-\delta) + \left[2 \frac{\left(|\mathcal{Y}| - 1\right)^p + \left(|\mathcal{A}| - 1\right)^p}{\left(2\right)^p} \right]^{\frac{1}{p}} \delta,  \\
        \leq{}& \sqrtp{2} \frac{\normp{\binom{|\mathcal{Y}|}{|\mathcal{A}|}-1}}{e^{\lambda \xi} + 1} (1-\delta) + \frac{\sqrtp{2}}{2} \normp{\binom{|\mathcal{Y}|}{|\mathcal{A}|}-1} \delta. 
    \end{align*} \\

    \noindent
    \customparagraph{Optimizing the softmax temperature.} Our goal is to minimize the right hand side of Equation~\eqref{eq:three_terms_bound}, we need to solve
    \begin{align*}
        \arg\inf\limits_\lambda \frac{2\sqrtp{2} \left\| \binom{|\mathcal{Y}|}{|\mathcal{A}|} - 1 \right\|_p}{e^{\lambda \xi} + 1}(1-\delta) + \lambda L (\abs{\mathcal{Y}}+\abs{\mathcal{A}})^{\abs{\frac{1}{2}-\frac{1}{p}}}I_W(Z_y, Z_a).
    \end{align*}
    Let $\alpha = 
    \sqrtp{2} \normp{\binom{|\mathcal{Y}|}{|\mathcal{A}|} - 1}(1-\delta)$ and $\beta = L(\abs{\mathcal{Y}}+\abs{\mathcal{A}})^{\abs{\frac{1}{2}-\frac{1}{p}}}I_W(Z_y, Z_a)$
    which are both positive values, then we consider
    \begin{align*}
        \arg\inf\limits_\lambda  \frac{2 \alpha}{e^{\lambda \xi} + 1} + \lambda \beta.
    \end{align*}
    Let $\gamma = \lambda \xi$, since $\xi > 0$ and $\alpha > 0$ then,
    \begin{align*}
        \arg\inf\limits_\lambda \frac{2\alpha}{e^{\gamma} + 1} + \lambda \beta = \frac{1}{\xi} \arg\inf\limits_\gamma \frac{1}{e^{\gamma} + 1} + \gamma \frac{\beta}{2\xi\alpha}.
    \end{align*}
     Let $c = \frac{\beta}{2\xi\alpha} \geq 0$ by definition, then we solve
    \begin{align*}
        \arg\inf\limits_{\gamma} \frac{1}{e^{\gamma} + 1} + c \gamma.
    \end{align*}
    We can study this function by looking at the sign of its derivative. 
    Considering the derivative equal to 0, we have
    \begin{align*}
        & c - \frac{e^{\gamma}}{(e^{\gamma} + 1)^2} = 0 \\
        \Leftrightarrow~ & c(e^{\gamma} + 1)^2 - e^{\gamma} = 0 \\
        \Leftrightarrow~ & ce^{2\gamma} + 2ce^{\gamma} + c - e^{\gamma} = 0 \\
        \Leftrightarrow~ & ce^{2\gamma} + (2c -1)e^{\gamma} + c = 0.
    \end{align*}
    With a change of variables $x=e^{\gamma}$, we solve
    \begin{align*}
        cx^2 + (2c -1)x + c = 0,
    \end{align*}
    and obtain the following root $\Delta = (2c-1)^2 - 4c^2 = 1-4c$. In the following, we consider two cases:
    
    \begin{itemize}
        \item Let $c\geq\frac{1}{4}$, then $\Delta \leq 0$ and there no or a single root. Since $c\geq0$, the gradient is always positive which implies that the minimum is reached at $\gamma = 0$ which is $\lambda=0$. Therefore, in this case, the bound is equal to 
        $\alpha = \sqrtp{2}\|\binom{|\mathcal{Y}|}{|\mathcal{A}|} - 1\|_p$.

        \item Let $c < \frac{1}{4}$, then $\Delta > 0$ and we have $x = \frac{1 - 2c \pm \sqrt{1-4c}}{2c}$. Since, $x=e^{\gamma}$ and $\gamma \geq 0$, then $x\geq1$. \\

        If $x = \frac{1 - 2c - \sqrt{1-4c}}{2c}$ and $x\geq1$, then $1-4c \geq \sqrt{1-4c}$ which is impossible since $c<\frac{1}{4}$.

        It implies that $x = \frac{1 - 2c + \sqrt{1 - 4c}}{2c}$ which is $\lambda = \frac{1}{\xi} \log(\frac{1 - 2c + \sqrt{1 - 4c}}{2c})$. Then, we have


        $$\lambda = \frac{1}{\xi} \log\left(\frac{1 - 2c + \sqrt{1 - 4c}}{2c}\right) 
            = \frac{1}{\xi} \log\left(\frac{1}{2c}\left(1+\sqrt{1-4c}\right) - 1 \right). $$

        Since we have an increasing function for $\lambda' \geq \lambda$ and $\sqrt{1-4c} \leq 1$, we can consider

        $$ \lambda \leq \lambda' = \frac{1}{\xi} \log \left( \frac{1}{c} -1 \right). $$ 

        In this case, the bound becomes
        \begin{align*}
            \frac{2 \alpha}{e^{\frac{1}{\xi}\log\left(\frac{1}{c} - 1\right)\xi}} + \frac{1}{\xi} \log\left(\frac{1}{c} - 1\right) \beta
            & = 2 \alpha \frac{1}{\frac{2\xi\alpha}{\beta}} + \frac{1}{\xi} \log \left(\frac{2\xi\alpha}{\beta} -1\right) \beta, \\
            &= \frac{\beta}{\xi} \left[1+\log \left( \frac{2\xi\alpha}{\beta} - 1\right)\right]. 
        \end{align*}







        Thus, we obtain the following bound
        \begin{align*}
            \min \left(\alpha, \frac{\beta}{\xi} \left[1+\log \left( \max\left(4, \frac{2\xi\alpha}{\beta}\right) - 1\right)\right]   \right),
        \end{align*}
        where the left term of the minimization corresponds to the bound when $\frac{2\xi\alpha}{\beta} \leq 4$, otherwise the bound is equal to the right term.
    \end{itemize}
\end{proof}

\section{Details of the bounds handling under different sensitive attribute scenarios.}
\label{app:scenarios_theorems_sa}

We summarize the different scenarios of sensitive attributes and how adaptable our theoretical framework is in Table \ref{tab:sa_scenarios_bound}.

\begin{table}[h]
    \centering
    \resizebox{\textwidth}{!}{%
    \begin{tabular}{l|cccc}
      \hline
      Type of SA & Lemma 1 \& 2 & Th. 3 & Th. 4 & 
      Empirically\\ \hline
      $A \in \{0, 1\}$ & \checkmark & \checkmark & \checkmark & \checkmark  \\ \cdashline{2-5}
      $A \in \{0,\dots,K\}$ & \checkmark & 
      \begin{tabular}{c}
           \checkmark (with Natarajan \\dimension)
      \end{tabular}
      & \begin{tabular}{p{40mm}}
        \checkmark (with generalization of the $\mathcal{H}\Delta\mathcal{H}$-divergence (Sicila et al., 2022))  
      \end{tabular} & \checkmark \\ \cdashline{2-5}
      Multiple SA  & \multicolumn{4}{c}{\checkmark (considering the intersection of Y and A so that $A \in \{0,\dots,K\}$)} \\ \cdashline{2-5}
      $A \in \mathbb
      R$ & \begin{tabular}{c} \checkmark (with \\ binning) \end{tabular} & \begin{tabular}{c} \checkmark (with \\ binning) \end{tabular} & \begin{tabular}{c} \checkmark (with \\ binning) \end{tabular} & \begin{tabular}{c} \checkmark (with \\ regression) \end{tabular} \\
      \hline
    \end{tabular}}
    \caption{Summary of the different possible scenarios for the sensitive attributes (SA)}
    \label{tab:sa_scenarios_bound}
\end{table}

\section{Experimental details}
\label{app:experimental_details}
\subsection{\wfc~ algorithm}
\label{app:algorithm}
In this section, we describe the full algorithm of \wfc. Algorithm~\ref{alg:us} provides the detailed algorithm for \wfc\ used in our experiments.

\begin{algorithm*}[h]
\caption{\wfc\ Algorithm}\label{alg:us}
\KwData{$D = \{(x_i,y_i,a_i)\}_{i=1}^n$ the training set, $n_e$ the number of epochs, $n_c$ and $n_d$ the number of training iterations per epoch for the critic and the classifier respectively, a batch size $n_b$, two neural networks $h_a(Enc(x))$ and $h_y(Enc(x);\theta)$, 
 a Critic $C_{\omega}$, a weight on the regularization $\beta$}
\For{e = 1, ..., $n_e$}{
    \For{t = 1, ..., $n_c$}{
        Sample $\{x_i,y_i,a_i\}_{i = 1}^{n_b}$
        
        Encode :  $z_{a} \leftarrow \{h_a(Enc(x_i))\}_{i=1}^{n_b}$, $z_{y} \leftarrow \{h_y(Enc(x_i))\}_{i=1}^{n_b}$
        
        Concatenate vectors to get $Z_{dep} \leftarrow [z_{a,i},z_{y,i}]_{i = 1}^{n_b}$
        
        Shuffle the $z_{a,i}$ vectors.
        
        Concatenate vectors to get $Z_{ind} \leftarrow [z_{s,i},z_{y,i}]_{i = 1}^{n_b}$
        
        $grad(w) \leftarrow \nabla_{\omega} \frac{1}{n_b}(\sum_{i=1}^{n_b} C_{\omega}(Z_{dep,i}) - \sum_{i=1}^{n_b} C_{\omega}(Z_{ind,i}))$
        
        $\omega \leftarrow Adam(\omega;grad(w))$
    }
    
    \For{t = 1, ..., $n_d$}{
        Sample $\{x_i,y_i,a_i\}_{i = 1}^{n_b}$
    
        Encode : $z_{s} \leftarrow \{h_a(x_i)\}_{i=1}^{n_b}$, $z_{y} \leftarrow \{h_y(x_i)\}_{i=1}^{n_b}$
        
        Concatenate vectors to get $Z_{dep} = [z_{a,i},z_{y,i}]_{i = 1}^{n_b}$
        
        Shuffle the $z_{a,i}$ vectors.
        
        Concatenate vectors to get $Z_{ind} = [z_{a,i},z_{y,i}]_{i = 1}^{n_b}$
        
        $\mathcal{L} \leftarrow \sum_{i = 1}^{n_b} \mathcal{L}(y_i,h_y(Enc(x_{y,i})))$
        
        $\mathcal{L} \leftarrow \mathcal{L}  + \beta ( \sum_{i=1}^{n_b} C_{\omega}(Z_{dep,i}) - \sum_{i=1}^{n_b} C_{\omega}(Z_{ind,i}))$
        
        $\theta \leftarrow Adam(\theta;\nabla_{\theta} \frac{1}{{n_b}} \mathcal{L})$
    }
}
\end{algorithm*}

\subsection{Details when using BERT-encoder}
\label{app:archi_t1}

In this section, we provide additional experimental details, notably, we detail the architectures of the MLPs and give the optimal hyperparameters when BERT model is used to obtain the initial representations.

\subsubsection{MLP architecture}

In Table~\ref{fig:mlp}, we present the architectural details of the classifier MLP. We grid searched over the learning rate ($lr \in \{1e^{-5}, 1e^{-4}, 1e^{-3}, 5e^{-5}, 5e^{-4}, 5e^{-3}\}$, the number of training batches for classification per epoch $n_d \in \{5, 10, 20\}$, the value used to clip the weights to enforce the Lipschitz constraint $clipping~value \in \{0.001, 0.01, 0.1\}$, the parameter $\beta \in \{0.1, 0.5, 1, 2, 5, 10, 100\}$, the layer used between the \textit{first hidden}, \textit{last hidden}, or \textit{last} layer.
 
\begin{table}[h]
\begin{subtable}[h]{0.45\textwidth}
   \centering
\begin{tabular}{rcc}
\hline
          Data set & Bios & Moji    \\
                     \hline
input dimension & 768        & 2304    \\
hidden layers        & 2          & 2      \\
hidden dimension     & 300        & 300     \\
learning rate        & $1^{-4}$      & $1^{-5}$ \\
batch size           & 128        & 128     \\
epochs max           & 10000      & 10000   \\
activation           & TanH       & TanH    \\
$\beta$               & 1          & 1       \\
$n_c$                & 20          & 5       \\
$n_d$                 & 5          & 5      \\
clipping value          & 0.01       & 0.01    \\
layer used           & last      & last   \\ \hline
\end{tabular} 
\caption{Details on hyperparameters used for the classifying MLP.}
\label{fig:mlp} 
\end{subtable} 
\begin{subtable}[h]{0.45\textwidth}
    \centering
\begin{tabular}{rc}
\hline
Hyperparameter & Value \\
\hline
number hidden layer & 1 \\
hidden dimension & 512 \\
activation & ReLU \\
optimizer & \begin{tabular}{c}
     Root Mean Square \\ Propagation    
\end{tabular}  \\
learning rate & $5e^{-5}$ \\ \hline
\end{tabular}
\caption{Details on hyperparameters used for the Critic MLP.}
\label{fig:mlpc}
\end{subtable}
\caption{Hyperparameter details when using BERT-encoder.}
\end{table}

\subsubsection{Critic architecture}
 In Table~\ref{fig:mlpc}, we present the architectural details of the Critic, which is a simple multi-layer perceptron. We grid searched over the learning rate $lr \in \{5e^{-5},~5e^{-4},~5e^{-3}\}$.
 
\begin{table}[h]

\end{table}

\subsection{Details when using SFR-Embeddings-2\_R}
\label{app:details_task_4}
\subsubsection{MLP architecture}

In Table~\ref{fig:mlp_sfr}, we present the architectural details of the classifier MLP when the embeddings are produced by the SFR-Embeddings-2\_R. We grid searched over the learning rate ($lr \in \{3e^{-7},~3e^{-6},~3e^{-5},~3e^{-3},~3e^{-1}\}$, the number of training batches for classification per epoch and for the Critic training $n_d,~n_c \in \{5, 10, 20\}$, and the hidden layer dimension (100, 300, 900).

\begin{table}[h]
\begin{subtable}[h]{0.45\textwidth}
\centering
\begin{tabular}{rc}
\hline
Hyperparameter & Value \\
\hline
input dimension & 4096  \\
hidden layers        & 1      \\
hidden dimension     & 300      \\
learning rate        & $3e^{-5}$   \\
batch size           & 128       \\
epochs max           & 10000     \\
activation           & TanH   \\
$\beta$               & 1  \\
$n_c$                & 20  \\
$n_d$                 & 10 \\
clipping value          & 0.01 \\
layer used           & last \\
\hline
\end{tabular}
\caption{Details on hyperparameters used for the classifying MLP.}
\label{fig:mlp_sfr}
\end{subtable}
\begin{subtable}[h]{0.45\textwidth}
\centering
\begin{tabular}{rc}
\hline
Hyperparameter & Value \\
\hline
number hidden layer & 1 \\
hidden dimension & 512 \\
activation & ReLU \\
optimizer & \begin{tabular}{c}
     Root Mean Square \\ Propagation    
\end{tabular}  \\
learning rate & $3e^{-6}$ \\ \hline
\end{tabular}
\caption{Details on hyperparameters used for the Critic MLP.}
\label{fig:mlpc_sfr}
\end{subtable}
\caption{Hyperparameter details for SFR-Embeddings-2\_R.}
\end{table}

\subsubsection{Critic architecture}
In Table~\ref{fig:mlpc_sfr}, we present the architectural details of the Critic for the task using SFR-Embeddings-2\_R. We grid searched over the learning rate $lr \in \{3e^{-7},~3e^{-6},~3e^{-5},~3e^{-3},$ $~3e^{-1}\}$.

\subsubsection{Baselines hyperparameters}
We select the best hyperparameters for the baselines for the classification of the representations generated by the SFR-Embedding-2\_R model. Following \citet{shen2022does}, we first determine the optimal hyperparameters of the classification models and keep those hyperparameters fixed when searching for the method-specific best hyperparameters. 
We tune the learning rate (lr $\in~\{3e^{-1},~3e^{-2},~3e^{-3},~\boldsymbol{3e^{-4}},~3e^{-5}\}$ and the hidden dimension ($\in\{100,~\boldsymbol{300},~900\}$). For the ADV baseline, we take 3 adversaries and consider several values for the following hyperparameters 
\textit{adv\_diverse\_lambda} $\in \{1e^{-1},~1e^{-2},~\boldsymbol{1e^{-3}},~1e^{-4}\}$ and \textit{adv\_lambda} $\in \{0.3,~\boldsymbol{0.5},~1,~2\}$. Values in bold are the selected ones. When BTEO is used the hyperparameters are set to 'EO' for \textit{BTObj}, 'Resampling' for \textit{BT} as in \citep{shen2022does}.
Finally, the embedding size is 4096, the batch size is 1024, and we set a patience of 10 for the early stopping.

\subsubsection{Details for Cross-domain \wfc}
\label{app:cross_domain}

In this section, we explain how we build the data set used for the cross-domain experiment to increase the divergence with the Bios data set.
To do so, we remove a set of words from the MP data set with regard to the sensitive attribute: gender. The words included in the set are the following:
\textit{'he', 'him', 'his', 'himself', 'Mr.', 'Sir', 'Lord', 'King', 'Prince', 'man', 'boy', 'gentleman', 'father', 'son', 'husband', 'brother', 'uncle', 'nephew', 'king', 'prince', 'she', 'her', 'hers', 'herself', 'Mrs.', 'Ms.', 'Miss', 'Lady', 'Dame', 'Queen', 'Princess', 'woman', 'girl', 'lady', 'mother', 'daughter', 'wife', 'sister', 'aunt', 'niece', 'queen', 'princess'}. 

\newpage

\bibliography{sample}

\end{document}